\documentclass{article}

\usepackage{lipsum}
\usepackage{microtype}
\usepackage{graphicx}
\usepackage{subcaption}
\usepackage{booktabs} 
\usepackage[table,xcdraw]{xcolor}
\usepackage{multirow}

\usepackage{hyperref}
\usepackage{tcolorbox}
\usepackage{tikz,xcolor}
\usepackage{enumitem}

\usepackage[accepted]{icml2025}

\usepackage{amsmath}
\usepackage{amssymb}
\usepackage{mathtools}
\usepackage{amsthm}

\usepackage{afterpage}
\usepackage{xfrac}

\usepackage[capitalize,noabbrev]{cleveref}

\theoremstyle{plain}
\newtheorem{theorem}{Theorem}[section]
\newtheorem{proposition}[theorem]{Proposition}
\newtheorem{lemma}[theorem]{Lemma}
\newtheorem{fact}[theorem]{Fact}
\newtheorem{corollary}[theorem]{Corollary}
\theoremstyle{definition}

\theoremstyle{remark}

\newtheorem{example}[theorem]{Example}

\newcommand{\level}{y}

\newcommand{\quant}{Q}
\newcommand{\dequant}{Q^{\dagger}}
\newcommand{\first}{1st }
\newcommand{\second}{2nd }
\newcommand{\argmin}{\mathop{\text{argmin}}\limits}

\newcommand{\ie}{\textit{i.e.}}
\newcommand{\eg}{\textit{e.g.}}

\definecolor{highlight}{HTML}{B3002A}
\definecolor{signedcolor}{HTML}{C5DFB3}
\definecolor{unsignedcolor}{HTML}{BCD6ED}

\usepackage[textsize=tiny]{todonotes}

\icmltitlerunning{SOLO}

\begin{document}

\twocolumn[
\icmltitle{
    Pushing the Limits of Low-Bit Optimizers: A Focus on EMA Dynamics
}




\begin{icmlauthorlist}
\icmlauthor{Cong~Xu}{ecnu}
\icmlauthor{Wenbin~Liang}{ecnu,youtu}
\icmlauthor{Mo Yu}{wechat}
\icmlauthor{Anan~Liu}{teg}
\icmlauthor{Ke-Yue~Zhang}{youtu}
\icmlauthor{Shunli~Wang}{youtu}
\icmlauthor{Lizhuang~Ma}{ecnu}
\icmlauthor{Jianyong~Wang}{thu}
\icmlauthor{Jun~Wang}{ecnu}
\icmlauthor{Wei~Zhang}{ecnu,sii}
\end{icmlauthorlist}

\icmlaffiliation{ecnu}{East China Normal University}
\icmlaffiliation{wechat}{WeChat AI, Tencent}
\icmlaffiliation{thu}{Tsinghua University}
\icmlaffiliation{youtu}{Tencent Youtu Lab}
\icmlaffiliation{teg}{Machine Learning Platform Department, Tencent TEG}
\icmlaffiliation{sii}{Shanghai Innovation Institute}

\icmlcorrespondingauthor{Cong Xu}{congxueric@gmail.com}
\icmlcorrespondingauthor{Jun Wang}{wongjun@gmail.com}
\icmlcorrespondingauthor{Wei Zhang}{zhangwei.thu.2011@gmail.com}

\begin{center}
   \small \url{https://github.com/MTandHJ/SOLO}
\end{center}

\icmlkeywords{Optimization, Quantization, LLM}

\vskip 0.3in
]



\printAffiliationsAndNotice{}  

\begin{abstract}
    The rapid scaling of models has led to prohibitively high training and fine-tuning costs.
    A major factor accounting for memory consumption is the widespread use of stateful optimizers (\eg, Adam), 
    which maintain auxiliary information of even 2x the model size in order to achieve optimal convergence.
    We therefore present SOLO in this work to spawn a novel type of optimizer that requires an extremely light memory footprint.
    While previous efforts have achieved certain success in 8-bit or 4-bit cases,
    SOLO enables Adam-style optimizers to maintain quantized states with precision as low as 3 bits, or even 2 bits.
    This immense progress is due to the identification and resolution of two key challenges:
    the signal swamping problem in unsigned quantization that results in unchanged state dynamics,
    and the increased gradient variance in signed quantization that leads to incorrect descent directions.
    The theoretical analysis suggests a tailored logarithmic quantization for the former 
    and a precision-specific momentum hyperparameter for the latter.
    SOLO can thus be seamlessly applied to Adam-style optimizers, leading to substantial memory savings 
    with minimal accuracy loss.
\end{abstract}

\section{Introduction}

\begin{figure}
	\centering
	\includegraphics[width=0.47\textwidth]{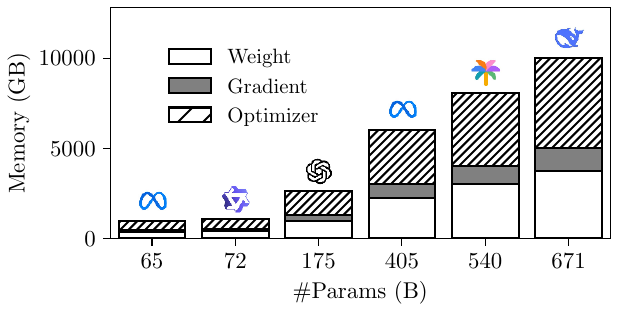}
    \caption{
        Memory costs of frontier LLMs, from left to right including
        \citet{touvron2023llama,yang2024qwen2,radford2018gpt,genai2023llama,anil2023palm,liu2024deepseekv3}.
        The costs are estimated based on a standard mixed-precision training:
        1) 16-bit model weights along with a 32-bit copy;  
        2) 16-bit gradients;  
        3) 32-bit optimizer states.  
        Other activations and temporary buffers are excluded for simplicity.  
    }
	\label{fig-memory-wgo}
    \vspace{-0.5cm}
\end{figure}

The rapid growth in model size has begun to strain the ongoing advancement of AI across various domains 
(\eg, Computer Vision~\cite{dosovitskiy2020vit,liu2023llava}, Natural Language Processing~\cite{radford2018gpt,genai2023llama}, and Recommender Systems~\cite{zhai2014hstu,zhang2024wukong}).
Fundamental AI research is, regrettably, becoming a luxury beyond the reach of most researchers and teams.
Recall that modern optimizers typically consume substantial memory to maintain auxiliary information aimed at optimizing and accelerating convergence.
For instance, the commonly adopted Adam(W) optimizer~\cite{kingma2014adam,loshchilov2017adamw}
necessitates a buffer of 2x model size to track 
\first moment (for approximating the signed `global' descent directions) 
and \second moment (for estimating the unsigned adaptive learning rates).
As illustrated in \figurename~\ref{fig-memory-wgo}, 
the maintained states account for nearly half of the total memory consumption in a standard mixed-precision training~\cite{micikevicius2017mixed}, doubtlessly the critical memory bottleneck in giant model training.

Once the optimizer is slimmed down, the saved memory can be reallocated to accommodate a larger model or increased batch size.
There are some straightforward \textit{hardware-level} methods to accomplish this, 
such as GPU sharding~\cite{rajbhandari2020zero} and CPU offloading~\cite{ren2021zerooffload};
however, these solutions are known to struggle with excessive communication overloads~\cite{wang2023zero++}, 
and are not easy to use in practice.
Low-bit optimizers~\cite{dettmers2015de}, which maintain states at low precision, 
stand out due to their orthogonality to the \textit{model architecture}, \textit{task}, \textit{hardware}, and the aforementioned \textit{distributed algorithms}.
For example, about $\sfrac{1}{8}$ of the memory overhead is enough for an optimizer 
if the states are maintained in 4-bit instead of the original 32-bit format.
Previous efforts have achieved certain success in 8- or 4-bit quantization~\cite{dettmers20228bit,li20234bit},
but they just apply regular quantization tricks to minimize the quantization error,
overlooking the peculiarities of the dynamic state update.
As a result, the existing 4-bit optimizers suffer from non-negligible performance degradation, 
let alone 3-bit or 2-bit precision.

Recall that the Exponential Moving Average (EMA) mechanism with a momentum $\beta \in [0, 1)$ 
is widely used in modern optimizers to update state $x$ upon a newly received signal $z$:
\begin{equation*}
   x_{t+1} \leftarrow \beta \cdot x_t + (1 - \beta) \cdot z_{t+1}.
\end{equation*}
We demonstrate that some overlooked problems exist when unsigned or signed quantization is directly applied to the EMA state update on the fly:
\begin{itemize}[leftmargin=*]
    \item Unsigned EMA update is susceptible to a signal \textit{swamping} problem~\cite{higham1993swamping,wang2018swamping};
    that is, the state might keep unchanged regardless of the gap between the newly received signal and the current state value.
    This problem becomes particularly severe when employing a large momentum ($\beta \rightarrow 1$), 
    where the added value of $(1 - \beta)z_{t+1}$ becomes very small and is thus easily swamped.
    Since the unsigned EMA update is typically employed to estimate adaptive learning rates,
    the useful adaptivity can no longer kick in once the swamping problem occurs.
    \item 
    Signed EMA update is highly susceptible to quantization errors due to the additional burden of sign representation and its direct impact on the update direction.
    Moreover, these quantization errors inherently increase the gradient noise variance, 
    leading to inevitable slowdowns in convergence and even training collapses.
\end{itemize}

Thanks to these deeper understandings,
we are to develop \textbf{S}tateful \textbf{O}ptimizers in ultra-\textbf{LO}w bits~(SOLO).
To the best of our knowledge, we are the first to point out the feasibility of reducing optimizer precision to as low as 3 or even 2 bits.
\begin{itemize}[leftmargin=*]
    \item 
    While stochastic rounding~\cite{xia2020sr} is the straightforward way to escape from strict signal swamping, 
    it is unfortunately accompanied by high variance when quantifying state values close to zero.
    In order to achieve robust unsigned state quantization, 
    more quantization levels should be allocated near the zero point rather than evenly distributed levels.
    Hence, the logarithmic quantization~\cite{oh2021logq} with a specific base is adopted.
    It additionally allows for an exact exponential decay with an easy-to-implement stochastic rounding mechanism.
    \item
    We derive an upper bound for the gradient noise variance when performing low-bit quantization in signed EMA updates.
    Since too high a gradient variance impairs training~\cite{li2023convergence,wang2024convergence},
    we hypothesize that an acceptable convergence rate could be achieved 
    once the upper bound is somewhat regulated.
    This suggests a direct adjustment of the momentum hyperparameter $\beta$
    for an acceptable gradient noise variance.
    For instance, in fine-tuning scenarios, 
    we recommend reducing $\beta$ from 0.9 to $\beta \le 0.820$ in the 4-bit setting and to $\beta \le 0.527$ in the 2-bit setting.
\end{itemize}

We apply SOLO to the widely used Adam(W) optimizer, giving two variants operating at ultra-low precision.
First, we propose the 4/2-bit Adam(W), which performs 4-bit signed and 2-bit unsigned quantization, respectively.  
This version is equivalent to a 3-bit Adam(W) optimizer in memory savings 
but allows more stable performance due to a more balanced bit allocation.
Second, we develop a 2-bit Adam(W) optimizer to explore the limits of ultra-low-bit optimizers.

The superiority of SOLO has been demonstrated across various domains:
(\textbf{CV}) Swin-T for image classification;  
(\textbf{NLP}) Transformer-Base for neural machine translation and RoBERTa-Large for natural language understanding;  
(\textbf{RS}) DCN for click-through rate prediction and HSTU-Large for sequential recommendation.
Finally, we validate the effectiveness in giant model training,
including (\textbf{LLM}) LLaMA-7B fine-tuning on Alpaca 
and (\textbf{LVM}) LLaVA-1.5 visual instruction tuning based on CLIP and Vicuna-7B-1.5.

This paper is organized as follows.  
We begin by reviewing related work, with a focus on efforts to alleviate the `hardware bottleneck'.
Subsequently, 
we delve into the details of low-bit optimizers and systematically analyze the challenges involved,
followed by discussions on feasible solutions.  
In Section~\ref{section-experiments}, 
we present a comprehensive comparison between SOLO and other low-precision counterparts, 
supported by extensive empirical analysis.  
For clarity, preliminaries on quantization methods and detailed theoretical proofs are provided in Appendix.

\section{Related Work}

\subsection{Efficient Architecture and Parallelism}

One of the most straightforward approaches is to redesign the base model~\cite{waswani2017transformer,radford2018gpt,gu2023mamba},
such as replacing the computationally intensive attention mechanism with alternatives like multi-head latent attention~\cite{liu2024deepseekv2} or linear attention~\cite{katharopoulos2020linearattention}, 
or employing Mixture-of-Experts (MoE) architectures~\cite{shazeer2017moe}.
This usually improves both training and inference efficiency 
but requires extensive engineering validation efforts.
There are some more parameter-efficient fine-tuning techniques~\cite{han2024peft},
such as soft-prompt on tuning the input~\cite{li2021prefix} and LoRA on low-rank adaptation~\cite{hu2021lora}.
Since only a limited number of parameters are involved in training, the required optimizer states can be significantly reduced.

Data and model parallelism have been widely employed for large-scale model training across multi-GPU clusters.
Recent studies have identified their limitations in memory reduction, 
prompting the development of specialized optimization techniques~\cite{seide20141bitsgd}. 
For instance, ZeRO~\cite{rajbhandari2020zero} enables efficient partitioning of optimizer states, gradients, and model parameters. 
Furthermore, ZeRO-Offload~\cite{ren2021zerooffload} and ZeRO-Infinity~\cite{rajbhandari2021zeroinfinity} 
extend this framework by leveraging CPU and NVMe resources, respectively.
However, these strategies are associated with non-negligible communication overhead, 
despite the presence of certain modifications~\cite{wang2023zero++}.

In summary, the aforementioned architectural designs or parallelism strategies are often dependent on specific models or hardware, 
thereby necessitating additional code refactoring.
The low-bit optimizers explored in this work are orthogonal to these approaches and demonstrate both flexibility and generalizability.

\subsection{Network Quantization}

While model training and inference are typically conducted at full precision (\ie, 32 bits) by default, significant efficiency gains can be achieved through the application of network quantization~\cite{ashkboos2023towards,zhao2024atom,frantar2022gptq}, enabling computationally intensive operations (\eg, matrix multiplication) to be executed at lower precision.
This approach has achieved remarkable success in the inference phase,
wherein the majority of weights can be maintained at a precision as low as 4 bits~\cite{lin2024awq,lin2025duquant} with minimal performance degradation by `smoothing' the activation outliers.
As a result, deploying LLMs on personal computers becomes feasible.

However, this is still difficult to apply to Fully Quantized Training (FQT) 
in a model-independent manner~\cite{zhou2017fqt,tang2022fqt,xi2024jetfire,xi2023training}.
While FP16/BF16 mixed-precision training~\cite{micikevicius2017mixed} has been extensively validated, 
the appropriate approach for lower-precision quantization has yet to be thoroughly established.
Indeed, the FP8 mixed-precision training has only recently been realized in giant model training~\cite{liu2024deepseekv3}.
Current 4-bit~\cite{sun2020ultra} and even 1-bit~\cite{gao20241bit} FQT methods remain far from practical 
due to the unacceptable loss in accuracy.

\subsection{Lightweight Optimizers}

Lightweight optimizers have garnered increasing attention~\cite{shazeer2018adafactor,dettmers20228bit,luo2025badam} as traditional stateful optimizers are memory-intensive.
Most alternatives aim to simplify the Adam optimizer~\cite{kingma2014adam} by condensing the \second state~\cite{anil2019sm3}, 
namely the adaptive learning rates on the fly.
For instance, Adafactor~\cite{shazeer2018adafactor} introduces factored \second moment estimation, 
while the recently Adam-mini~\cite{zhang2024adammini} observes most entries within a block could share a common learning rate (\ie, the same unsigned state).
In addition, GaLore~\cite{zhao2024galore} attains memory-efficient optimization by conducting state updates in a low-dimensional space.
Lion~\cite{chen2024lion} relies solely on \first sign momentum to achieve comparable performance across most tasks.

Although these optimizers are lightweight by design, 
their generalizability has not been thoroughly validated compared to the widely adopted Adam(W) optimizer.
We therefore focus on low-bit optimizers due to their model-, task-, and hardware-agnostic characteristics.
They enable seamless deployment following the same procedures as their high-precision counterparts,
while substantially reducing the effort required for hyperparameter tuning compared to the aforementioned lightweight optimizers.  
Moreover, some lightweight optimizers like Adafactor can also benefit from low-bit quantization, further reducing memory consumption.
Currently, 8-bit~\cite{dettmers20228bit} and 4-bit~\cite{li20234bit} optimizers have been developed, 
primarily relying on regular quantization tricks such as linear or dynamic exponent mappings. 
In the following, 
we analyze how the overlook of the EMA update mechanism hinders the development of ultra-low-bit optimizers.

\section{Ultra-Low-Bit Optimizers}
\label{section-solo}

It is well known that the auxiliary information maintained by an optimizer contribute substantially to the overall memory cost.
For instance, the commonly adopted Adam(W) optimizer~\cite{kingma2014adam,loshchilov2017adamw}
necessitates a buffer of 2x model size to track 
\first moment (for approximating the signed `global' descent directions) and \second moment (for estimating the unsigned adaptive learning rates),
about 520 GB memory usage for a 65B model.
In the following, 
we first elucidate the effectiveness of the low-bit optimizer for addressing this problem, 
and then systematically analyze the challenges and corresponding solutions when extending to ultra-low-precision scenarios.

\subsection{Quantization and Dequantization in EMA Update}

The $b$-bit quantization for a high-precision tensor $X$
aims to map each element $x \in X$ to an integer $q \in \{0, 1, \ldots, 2^{b} - 1\}$.
Given the scale factor $\Delta := \max_{x \in X} |x|$, 
a majority of quantizers can be unified as a \textit{nearest rounding} process:
\begin{align}
    \label{eq-quantization}
    q = \quant(x) := \argmin_{k=0}^{2^b - 1} \big|\frac{x}{\Delta} - \level_k \big|,
\end{align}
where $\level_k, \: k=0,1, \ldots, 2^b - 1$ 
are \textit{monotonic} quantization levels distributed across the interval [0, 1] ([-1, 1]) for unsigned (signed) quantization.
For instance, linear quantization is a simple yet effective method, utilizing evenly distributed levels.
Intuitively, the resulting quantized value $q$ implies the normalized value $x / \Delta$ is closest to the $q$-th (indexed from 0) level $\level_q$.
These low-precision integers, 
along with a full-precision scale factor, 
can be stored and subsequently utilized to `recover' the high-precision tensor:
\begin{align*}
    \tilde{x} = \dequant(q) := \level_{q} \cdot \Delta.
\end{align*}
Compared to the original high-precision tensor,
the storage of quantized values benefits from reduced memory requirements, albeit at the expense of introducing some quantization (rounding) errors.

Recall that most recent (Adam-style) optimizers rely on the Exponential Moving Average (EMA) for state updates.
In the context of low-precision quantization, this process can be divided into three steps:
\begin{align}
   \text{Dequantization:  }  & \tilde{x}_t = \dequant(q_t) = \level_{q_t} \cdot \Delta_t, \\
   \text{EMA update:  } & \hat{x}_{t+1} \leftarrow \beta \cdot \tilde{x}_t + (1 - \beta) \cdot z_{t + 1}, \\
   \text{Quantization:  } & q_{t+1} = \quant(\hat{x}_{t+1}),
\end{align}
where $z_{t+1}$ is the newly received signal at $t + 1$ update.
Taking the Adam optimizer as an example, 
$z_{t+1}$ corresponds to the gradient $g_{t+1}$ in the \textit{signed} state update 
and the squared gradient $g_{t+1}^2$ in the \textit{unsigned} state update.
We emphasize that the signed and unsigned EMA updates typically serve distinct roles in optimizers. 
Hence, we analyze the respective challenges and then propose feasible solutions,
beginning with unsigned updates to signed updates.

\subsection{Quantization for Unsigned EMA Update}\label{section-unsigned}

\subsubsection{Signal Swamping Problem}

Recall that the unsigned EMA update is commonly applied in the denominator to estimate adaptive learning rates.
Therefore, the primary goal is to ensure that the $\hat{x}_t$ closely follows the trend of the ground-truth state $x_t$.
Unfortunately, for certain quantization levels $\level_i$ with a large radius,
the \textit{swamping} problem~\cite{higham1993swamping,wang2018swamping}, 
frequently observed in large-to-small number addition, may arise severely.
In this case, 
the received signal is `swamped' by the current state value, 
potentially resulting in a constant state, 
which fundamentally undermines the intended purpose of applying EMA.

\begin{theorem}[Signal Swamping]
    \label{thm-dsp}
    Let $\level_{q_t}$ be the closest level to $q_t$, and its radius be defined by
    $$
    r = \frac{1}{2} \min 
    \Big \{ 
        |\level_{q_t} - \level_{q_t - 1}|, |\level_{q_t + 1} - \level_{q_t}|
    \Big \}.
    $$
    The quantized state $q_t$ remains unchanged if
    \begin{equation}
        \label{eq-swamped-radius}
        r \ge (1 - \beta) \cdot \Big|\frac{z_{t+1}}{\Delta_{t+1}} - \level_{q_t} \Big| + \Big| \frac{\Delta_t}{\Delta_{t+1}} - 1 \Big |.
    \end{equation}
\end{theorem}
In other words, if the radius around $\level_{q_t}$ is large enough, 
the state will keep constant even receiving a signal $z_{t+1}$ quite different to the state value $\hat{x}_t$.
Four key factors significantly increase the likelihood of the signal swamping problem:
\begin{itemize}[leftmargin=*]
    \item \textbf{Low precision}. 
    A lower precision generally results in a larger radius $r$
    since the entire interval must be covered with fewer quantization levels. 
    \item \textbf{Large tensor size}. 
    For a larger-size tensor $X$ quantized as a whole,
    the scale factor $\Delta = \max_{x \in X} |x|$ becomes larger,
    potentially leading to a smaller value of $|z_{t+1} / \Delta_{t+1} - \level_{q_t}|$.
    Moreover, changes from $\Delta_t$ to $\Delta_{t+1}$ tend to be negligible, further exacerbating this issue.
    \item \textbf{Unsigned EMA update}.
    Note that for the signed state update,
    the accumulation process incorporates both positive and negative values.
    This results in a mixed state with a much smaller magnitude compared to the signal itself.
    In this case, the swamping problem is less severe due to a large $|z_{t+1} / \Delta_{t+1}|$ allowed.
    Conversely, the unsigned EMA update is more prone to signal swamping.
    \item \textbf{Momentum $\beta \rightarrow 1$}.
    A large momentum $\beta$ has a direct impact on the swamping problem.
    Notably, the unsigned EMA state update typically employs a larger momentum (\eg, $\beta = 0.999$) compared to the signed EMA update (\eg, $\beta = 0.9$), 
    which further exacerbates the issue in the unsigned EMA update.
\end{itemize}
To summarize,  
\begin{tcolorbox}
    The \textcolor{highlight}{unsigned EMA update}, particularly when employing a \textcolor{highlight}{large momentum} ($\beta \rightarrow 1$), is highly susceptible to signal swamping, especially in scenarios involving \textcolor{highlight}{low precision} and \textcolor{highlight}{large-size tensors}.
\end{tcolorbox}

\textbf{The impact of signal swamping.}
The EMA update serves as a dynamic estimation of the expectation $\mathbb{E}[z_t]$ 
when $z_t$ follows a stationary distribution.
Clearly a state undergoing signal swamping cannot be expected to converge to the true expectation value,
not to mention the complex state dynamics in practice.
In other words, once the swamping problem arises, 
the cumulative approximation error will inevitably grow to an unacceptable level.
To provide a clearer understanding of this impact,
we demonstrate several surprising cases associated with two widely adopted quantization methods.
As the unsigned EMA state update is most affected, 
we present the analysis on \textit{unsigned} quantization 
and leave a generalized version in Appendix~\ref{appendix-proofs}.

\begin{figure}
	\centering
	\includegraphics[width=0.45\textwidth]{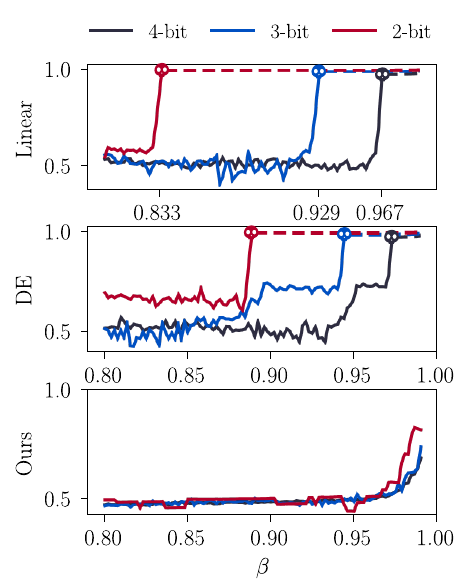}
	\caption{
        Synthetic uniform signals.
        $\mathbb{E}[z] = 0.5$ represents the expected value to which the quantized state shall converge.
	}
	\label{fig-rand-synthetic}
\end{figure}

\begin{corollary}
    \label{corollary-linear-de}
    If $\Delta_{t+1} = \Delta_t$ and $r_t \ge 2 (1 - \beta)$, 
    the unsigned quantized state $q_t$ remains unchanged for any $z_{t+1} \le \Delta_t$.
    To be specific,
    \begin{itemize}[leftmargin=*]
        \item For linear quantization, 
        the received signals are swamped by states corresponding to any level when
        \begin{align}
            \beta \gtrapprox 
            \underbrace{0.999}_{\text{8-bit}}, 
            \underbrace{0.967}_{\text{4-bit}}, 
            \underbrace{0.929}_{\text{3-bit}}, 
            \underbrace{0.833}_{\text{2-bit}}.
        \end{align}
        \item 
        For dynamic exponent quantization, 
        the received signals are swamped by states corresponding to half of the levels when
        \begin{align}
            \beta \gtrapprox 
            \underbrace{0.998}_{\text{8-bit}},
            \underbrace{0.966}_{\text{4-bit}},
            \underbrace{0.933}_{\text{3-bit}}, 
            \underbrace{0.837}_{\text{2-bit}}.
        \end{align}
    \end{itemize}
    Here `$\gtrapprox$' omits negligible rounding errors.
\end{corollary}

\begin{figure*}
	\centering
	\includegraphics[width=0.97\textwidth]{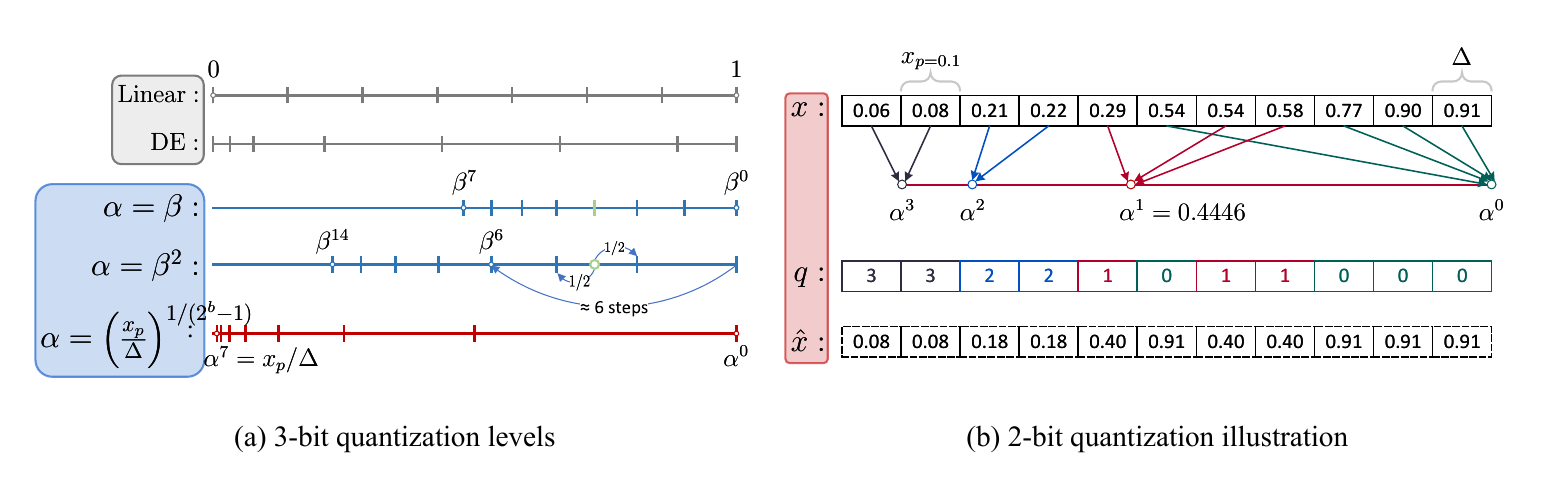}
	\caption{
        (a) Comparison of various quantization approaches:
        The proposed allocates a greater number of quantization levels near the zero point, effectively reducing variance.  
        (b) Illustration of 2-bit unsigned quantization: 2-digit precision number is shown for clarity.
	}
	\label{fig-qema}
\end{figure*}

Corollary~\ref{corollary-linear-de} depicts an extreme case 
in which the majority of signals are swamped, causing the corresponding states to remain seemingly constant.
The following example, based on synthetic data, provides further empirical evidence 
without the fixed $\Delta_{t}$ condition.

\begin{example}[Synthetic uniform signals]
    Consider $X \in \mathbb{R}^{1000}$, 
    where both the initial states and subsequent received signals are independently drawn from a uniform distribution $\mathcal{U}[0, 1]$. 
    \figurename~\ref{fig-rand-synthetic} depicts the converged state values after 100 iterations.
\end{example}

It is evident that both linear and dynamic exponent quantization fail to converge to the expected value of $0.5$ for a slightly large tensor 
when the momentum parameter $\beta \rightarrow 1$ is adopted.
This also explains why the previous approaches~\cite{dettmers20228bit,li20234bit} relied heavily on block-wise quantization.
In contrast, the solution proposed below exhibits exceptional robustness, 
making it more suitable for the unsigned EMA update.

\subsubsection{Our Unsigned EMA State Quantization}


\paragraph{Stochastic rounding.}
The swamping problem would not exist any more
if we apply \textit{stochastic rounding}~\cite{xia2020sr} instead of the nearest rounding method used in the Eq.~\eqref{eq-quantization}.
Without loss of generality, for a $\level_{k-1} \le x / \Delta \le \level_k$, 
stochastic rounding suggests
\begin{equation}
    \label{eq-stochastic-rounding}
    Q_{sr}(x) :=
    \left \{
        \begin{array}{ll}
            k-1, & w.p. \quad \frac{\level_k - x / \Delta}{ \level_k - \level_{k-1}}, \\
            k, & w.p. \quad \frac{x / \Delta - \level_{k-1}}{ \level_k - \level_{k-1}}.
        \end{array}
    \right .
\end{equation}
As a result, even receiving a very small signal $z_{t+1}$,
it could benefit from a non-zero probability,
thereby escaping from the swamped state.
Unfortunately, this is often accompanied by high variance when quantizing a state value close to zero,
as the unsigned EMA state typically appears in the denominator for estimating adaptive learning rates.
To be specific, the stochastic rounding in Eq.~\eqref{eq-stochastic-rounding} demonstrates a variance as follows:
\begin{proposition}
    \label{prop-unsigned-variance}
    For a \second state value $\level_{k-1} \le x / \Delta \le \level_{k}$ to be quantized,
    the corresponding adaptive learning rate has a variance:
    \begin{equation}
        \mathbf{Var}(\frac{1}{\sqrt{\hat{x}}}) 
        = \mathcal{O} \bigg( \Big(
            \frac{\Delta}{\sqrt{\level_k}} - \frac{\Delta}{\sqrt{\level_{k-1}}}
        \Big)^2 \bigg).
    \end{equation}
\end{proposition}

Consider a triplet $0 \le \level_{k-1} < \level_{k} < \level_{k+1}$ with similar spacing,
namely $|\level_{k} - \level_{k-1}| \approx |\level_{k + 1} - \level_{k}|$, 
The corresponding learning rates however exhibit greater differences 
as $\level_{k-1}$ approaches zero.
In other words, due to the denominator role of unsigned states, 
the same quantization errors can lead to significant differences in learning rates, 
making it difficult to track the true state dynamics for both linear and dynamic exponent quantization since their quantization levels are nearly evenly distributed.
To mitigate this issue, more quantization levels should be allocated near the zero point to provide compensation. 
This justifies that the commonly used linear or dynamic exponent quantization methods may not be well-suited in this scenario, 
as they fail to adequately account for this consideration.

As shown in \figurename~\ref{fig-qema},
we therefore adopt the logarithmic quantization with a specific base $\alpha \in (0, 1)$,
where the levels\footnote{
Notably, unlike the convention, 
the levels here are arranged in a monotonically decreasing order for the sake of notational simplicity.
Nevertheless, the aforementioned conclusions remain valid and applicable.
}
become $\level_k = \alpha^k, \: k=0,\ldots, 2^b-1$.
In addition, we apply a rounding mechanism slightly different to Eq.~\eqref{eq-stochastic-rounding}:
\begin{align}
    Q(x) 
    &\notag = \argmin_{k=0}^{2^b - 1} \big|\frac{x}{\Delta} \cdot \alpha^\xi - \level_k \big| \\
    &=\text{Clip}(\lfloor \log_{\alpha} \frac{x}{\Delta} + \xi \rceil; 0, 2^b - 1) \label{eq-unsigned-solo},
\end{align}
where $\xi \sim \mathcal{U}[-0.5, 0.5]$ and $\lfloor \cdot \rceil$ denotes the convergent rounding used in IEEE 754 floating-point operations.
This stochastic rounding mechanism is preferred here because
1) It is straightforward to implement, as it eliminates the need to calculate a level-specific probability for each state value;
2) It allows for an exact exponential decay when consecutive zero signals are received, aligning with the real state decay:

\begin{proposition}[Consecutive zero signals]
    \label{prop-state-decay}
    For a state value that starts with $\beta^{ck}, \: c, k \in \mathbb{N}_+$, it takes $c \cdot s$ iterations to decay to $(\beta^c)^{k + s}$ 
    if the subsequent signals $z \equiv 0$.
    For the logarithmic quantization with a base $\alpha=\beta^c$,
    the average number of iterations required to decay from $(\beta^c)^k$ to $(\beta^c)^{k + s}$ is given by
    \begin{equation}
        \mathbb{E}\big[(\beta^c)^k 
        \xrightarrow{\text{Ours}}
        (\beta^c)^{k+s} \big| z \equiv 0
        \big] = cs,
    \end{equation}
    which aligns with the real state decay.
\end{proposition}
It is worth noting that aligning this state decay is of particular value 
if parameters are partially involved during training.
This situation is frequently encountered during the construction of massive embedding tables, 
which serve as fundamental components in both natural language processing and recommender systems.
Because only a subset of embeddings is sampled during each mini-batch training iteration, 
the resulting gradient tends to be highly sparse.
In addition, some operations such as dropout can also result in zero gradients.

Remark that the standard logarithmic quantization typically adopts a base of $\alpha = 1/2$. 
This choice may cause excessive truncation of small states, leading to a slower convergence rate. 
Therefore, the preferred base is given as $\alpha = (x_{p} / \Delta)^{1 / (2^b - 1)}$
with $x_p$ being the $p$-quantile over the entire tensor.
In contrast to the fixed base, 
it adaptively covers the interval which most states are located in, 
thereby enabling it to effectively mimic the true dynamics through an appropriate value of $p$.
Empirically, a value of $p \in [0.05, 0.3]$ has been found to be effective across various scenarios from small to giant models.
\figurename~\ref{fig-qema}b illustrates a toy example of 2-bit unsigned logarithmic quantization.

\subsection{Quantization for Signed EMA Update}\label{section-signed}

\begin{figure}
    \begin{subfigure}{0.47\textwidth}
        \centering
        \includegraphics[width=0.88\textwidth]{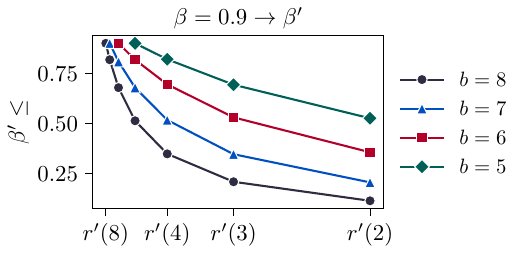}
    \end{subfigure}

    \hfill

    \begin{subfigure}{0.47\textwidth}
        \centering
        \begin{tabular}{l|
            >{\columncolor[HTML]{C0C0C0}}l
                ll
            >{\columncolor[HTML]{C0C0C0}}l}
            \toprule
                & $b=8$   & $b=7$   & $b=6$   & $b=5$   \\
            $b'=4$ & 0.350 & 0.518 & 0.695 & 0.820 \\
            $b'=3$ & 0.211 & 0.348 & 0.531 & 0.694 \\
            $b'=2$ & 0.116 & 0.207 & 0.357 & 0.527 \\
            \bottomrule
        \end{tabular}
    \end{subfigure}
	\caption{
        The upper bound of $\beta'$ for DE according to Eq.~\eqref{eq-beta-prime}.
	}
	\label{fig-beta}
\end{figure}

As discussed earlier, 
the signed EMA state quantization is only minimally impacted by the signal swamping problem due to its signed nature 
and typically small momentum value $\beta$.
Nonetheless,
it presents two distinct challenges in comparison to unsigned EMA state quantization:
\begin{itemize}[leftmargin=*]
    \item Signed quantization must allocate one bit to represent the sign, which inevitably amplifies quantization errors under the same precision constraint.
    \item The signed state demonstrates significantly higher sensitivity to quantization errors.
    A substantial quantization error could potentially reverse the descent direction to an ascent direction, significantly impairing convergence.
\end{itemize}
As a result, training failures are frequently observed in low-precision quantization, particularly in cases where training begins from scratch.
To summarize,
\begin{tcolorbox}
    The \textcolor{highlight}{signed EMA update}, 
    in particular for \textcolor{highlight}{training from scratch}, 
    is highly susceptible to quantization errors 
    due to the additional burden of \textcolor{highlight}{sign representation} and its direct impact on the \textcolor{highlight}{update direction}.
\end{tcolorbox}

\textbf{Quantization errors amplify the variance of gradients.}
Here, we provide a theoretical justification for why the signed EMA update is susceptible to quantization errors, 
which in turn inspires a feasible solution introduced later.
Firstly note that the received signal $z = g$ for signed EMA updates represents the \textit{stochastic} gradient over a mini-batch. 
Some prior studies~\cite{li2023convergence,wang2024convergence} have demonstrated that reducing gradient variance contributes to the convergence.
However, the quantization error can be treated as a component of gradient stochasticity, indeed increasing the variance of gradients.

\begin{theorem}
    \label{thm-first-variance}
    Let the maximum radius under $b$-bit stochastic quantization be
    \begin{equation*}
    r_{\max}(b) = \max_{k=1}^{2^b - 1} \frac{1}{2} |\level_{k} - \level_{k - 1}|.
    \end{equation*}
    The quantized EMA state update of the stochastic gradient $g$ is equivalent to a standard EMA state update:
    \begin{equation}
    x_{t+1} \leftarrow \beta \cdot x_t + (1 - \beta) \cdot \tilde{g}_{t + 1},
    \end{equation}
    where $\tilde{g}_{t+1}$ denotes the stochastic gradient with
    \begin{align}
    \mathbb{E}[\tilde{g}_{t+1}] = \mathbb{E}[g_{t+1}], \\
    \mathbf{Var}[\tilde{g}_{t+1}] \lessapprox \mathbf{Var}[g_{t+1}] + \Big(\frac{\beta}{1 - \beta} r_{\max}(b) \cdot \Delta_t \Big)^2.
    \end{align}
\end{theorem}

Theorem~\ref{thm-first-variance} suggests that
as the spacing between the quantization levels decreases or the momentum hyperparameter $\beta$ increases, 
the additional gradient variance is progressively introduced.
The former implies that the high gradient variance is almost inevitable in the ultra-low-bit scenario,
where large spacing arises due to the insufficient number of levels.
Clearly, the breakthrough here does not lie in the design of a quantization scheme.
SOLO thus turns to a modified momentum hyperparameter $\beta'$ 
to regulate the variance introduced by quantization within an acceptable range.
If the model converges properly under $b$-bit quantization with a momentum hyperparameter $\beta$,
we hypothesize that a similar convergence rate can be achieved with a modified combination of $(\beta', b')$ subject to
\begin{equation}
    \label{eq-beta-prime}
    \frac{\beta'}{1 - \beta'} r_{\text{median}}(b')
    \le \frac{\beta}{1 - \beta} r_{\text{median}}(b).
\end{equation}
It is worth noting that $r_{\text{median}}$ is adopted in place of $r_{\max}$ 
because the ratio of $r_{*}(b) / r_{*}(b')$ is nearly identical for both the median and maximum radius
while the former demonstrates greater stability under ultra-low-bit cases.
\figurename~\ref{fig-beta} illustrates the $\beta'$ allowed for unsigned DE quantization.

Identifying the least bits for a given momentum value $\beta$ can be time-consuming. 
Fortunately, we find that the $\beta'$ corresponding to the column of $b = 5$ is sufficient to maintain accuracy in most fine-tuning scenarios.
This suggests that a $\beta' \le 0.820$ and $\beta' \le 0.527$ can serve as alternatives to $\beta = 0.9$ in the 4-bit and 2-bit settings, respectively.
In addition, training from scratch is more susceptible to quantization errors in comparison to fine-tuning, 
and thus we recommend a lower $\beta'$ corresponding to the column of $b=8$ should be considered.

\subsection{Application: 4/2-bit and 2-bit Adam(W)}

Since Adam(W)~\cite{kingma2014adam,loshchilov2017adamw} is the most widely adopted optimizer for training and fine-tuning, 
we introduce in detail two SOLO-quantized variants of Adam(W).
\begin{itemize}[leftmargin=*]
    \item \textbf{4/2-bit Adam(W)} applies a 4-bit DE quantization to the signed EMA update, 
    while a 2-bit logarithmic quantization, as introduced in Eq.~\eqref{eq-unsigned-solo}, to the unsigned EMA update.
    Notably, it yields memory savings equivalent to a 3-bit Adam(W) while offering additional advantages:
    1) As previously discussed, the signed EMA update is more susceptible to quantization errors and should therefore be allocated more bits if possible;
    2) Technically, until the work done, mainstream deep learning frameworks have only provided official low-bit storage up to 8-bit precision.
    To this end, two 4-bit entries or four 2-bit entries should be packed into a single 8-bit element for storage,
    whereas such straightforward packing is not feasible for 3-bit representations.
    In addition, based on \figurename~\ref{fig-beta} and empirical observations, 
    a reduced value of $\beta \le 0.820$ for fine-tuning and $\beta \le 0.350$ for training from scratch generally ensures a stable training process. 
    In default, we fix the momentum of the \textit{signed} EMA update at 0.8 for fine-tuning and 0.3 for training from scratch,
    while keeping the momentum for the \textit{unsigned} EMA update unchanged.
    \item \textbf{2-bit Adam(W)} applies 2-bit quantization to both the signed and unsigned EMA updates, 
    representing the extreme case explored in this study.
    Note that this implies the 32-bit state values would be quantized into only four levels;
    the signed state encounters an additional challenge, as one bit must be allocated for sign representation, thereby reducing the effective precision.
    Analogous to the 4/2-bit Adam(W), unless otherwise specified, 
    we fix the momentum of the \textit{signed} EMA update at 0.5 for fine-tuning and 0.1 for training from scratch, 
    while keeping the momentum of the \textit{unsigned} EMA update unchanged.
    Regarding the $p$-quantile used in the tailored logarithmic quantization, we fix $p = 0.1$ for simplicity.
\end{itemize}
Following previous low-bit optimizers~\cite{dettmers20228bit,li20234bit}, 
block-wise quantization with a block size of 128 is adopted. 
This approach splits the entire tensor into blocks and performs quantization separately on each block.
However, we emphasize that the developed unsigned logarithmic quantization is robust to the choice of block size, 
whereas this is not the case for other approaches.

\section{Experiments}\label{section-experiments}

In this section,
we aim to investigate the superiority of SOLO by comparing it with its full-precision counterparts as well as other low-bit alternatives.

\textbf{Tasks.}
Since low-bit optimizers are model- and task-agnostic techniques, 
they are expected to be applicable across a wide range of scenarios.  
Here, we consider three primary application domains, Computer Vision (CV), Natural Language Processing (NLP), and Recommender Systems (RS),
to ensure a reliable and comprehensive evaluation. 
CV and NLP are tasks that attract significant interest from the research community, 
while RS represents a practical application receiving particular focus from the industrial sector.
For each domain, we initially employ some benchmarks that are relatively modest in scale by current standards.
In addition, we further investigate the performance of instruction tuning for both Large Language Model (LLM) and Large Vision Model (LVM).
Detailed descriptions of the models and datasets are provided in the Appendix~\ref{section-settings}.

\textbf{Baselines.}
We consider baselines ranging from full-precision to 2-bit optimizers.
The 16-bit Adam(W) optimizer is implemented using a standard BF16 format, 
which has been validated in the large-scale training of DeepSeek-v3~\cite{liu2024deepseekv3}.
The 8-bit Adam(W) variant~\cite{dettmers20228bit} employs the dynamic exponent quantization for both signed and unsigned cases. 
For the lower-precision 4-bit and 2-bit counterparts,
as suggested by~\cite{li20234bit}, 
a modified linear quantization scheme that excludes the zero point is employed for unsigned quantization 
while keeping the signed dynamic exponent quantization unchanged.
For a fair comparison, SOLO and other low-bit counterparts adopt block-wise quantization with a size of 128. 
However, as demonstrated in \figurename~\ref{fig-unsign-state-visual}, 
this constraint can be significantly relaxed for SOLO, whereas other approaches are highly sensitive to it.

\textbf{Implementation details.}
We implement SOLO-quantized and other low-bit optimizers based on the same PyTorch-based library, torchao\footnote{
    \url{https://github.com/pytorch/ao}
}. 
Training and fine-tuning protocols strictly follow the official codebase, 
with the only exception being a modified batch size to accommodate the available GPUs.
Most of the results reported below are averaged over three independent runs.
Given the robust performance and high computational demands of Swin-T and LLaVA-1.5, 
we limit our evaluation to a single run.
The specific hyperparameter settings are listed in Appendix~\ref{section-settings}.

\subsection{Overall Comparison}

\afterpage{
\begin{table*}[t]
    \begin{minipage}{0.36\textwidth}
        \setlength{\tabcolsep}{3pt}
        \centering
        \begin{tabular}{l|r|cc}
            \toprule
        \multicolumn{1}{l}{}                                                        & \multicolumn{1}{c|}{} & \multicolumn{2}{c}{Swin-T  \includegraphics[width=8pt]{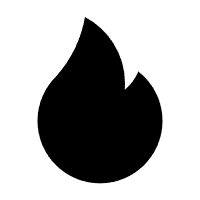} }                                                             \\
        \multicolumn{1}{l}{}                                                                           & \multicolumn{1}{l|}{} & \multicolumn{2}{c}{ImageNet-1K}                             \\
            \midrule
        \multicolumn{1}{l}{}                                                                           & \multicolumn{1}{l|}{} & \small ACC@1                        & \small ACC@5       \\
            \midrule
        \rowcolor[HTML]{EDEDED} 
                                                                                    & 32-bit Adam(W)         & 81.16                        & 95.55  \\
                                                                                    & 16-bit Adam(W)         & 80.73                        & 95.29  \\ 
                                                                                    & 8-bit Adam(W)          & 80.75                        & 95.43  \\ 
                                                                                    & 4-bit Adam(W)          & {\color[HTML]{FF0000} Crash} & {\color[HTML]{FF0000} Crash}    \\
                                                                                    & 2-bit Adam(W)          & {\color[HTML]{FF0000} Crash} & {\color[HTML]{FF0000} Crash}     \\
            \midrule
        \rowcolor[HTML]{DDEBF7} 
                                    & 4/2-bit Adam(W)        & 80.86                        & 95.48        \\
        \rowcolor[HTML]{DDEBF7} 
        \multirow{-2}{*}{\rotatebox{90}{\tiny \textbf{SOLO}}} & 2-bit Adam(W)          & 80.74                        & 95.39      \\
            \bottomrule
        \end{tabular}
        \subcaption{CV}
    \end{minipage}
    \hfil
    \begin{minipage}{0.32\textwidth}
        \setlength{\tabcolsep}{3pt}
        \centering
        \begin{tabular}{c|c}
            \toprule
        Transformer-B \includegraphics[width=8pt]{src/flame1.png}             & RoBERTa-L  \includegraphics[width=8pt]{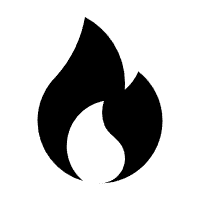} \\
        WMT                          & GLUE        \\
            \midrule
        \small BLEU                         & \small ACC/Corr    \\
            \midrule
        \rowcolor[HTML]{EDEDED} 
        27.45                        & 88.51         \\
        27.37                        & 88.72         \\
        27.50                        & 88.56         \\
        26.93                        & 88.55         \\
        {\color[HTML]{FF0000} Crash} & 88.31         \\
            \midrule
        \rowcolor[HTML]{DDEBF7} 
        27.31                        & 88.59         \\
        \rowcolor[HTML]{DDEBF7} 
        27.12                        & 88.66         \\
            \bottomrule
        \end{tabular}
        \subcaption{NLP}
    \end{minipage}
    \hfil
    \begin{minipage}{0.3\textwidth}
        \setlength{\tabcolsep}{3pt}
        \centering
        \begin{tabular}{cc|cc}
            \toprule
            \multicolumn{2}{c|}{DCN \includegraphics[width=8pt]{src/flame1.png}}    & \multicolumn{2}{c}{HSTU-L \includegraphics[width=8pt]{src/flame1.png}}                               \\
            \multicolumn{2}{c|}{Criteo } & \multicolumn{2}{c}{ML-1M} \\
            \midrule
            \small AUC          & \small LOGLOSS\textdownarrow     & \small HR       & \small NDCG     \\
            \midrule
        \rowcolor[HTML]{EDEDED} 
            81.38       & 0.43812     & 32.73      & 18.80      \\
            81.06       & 0.44111     & 32.83      & 18.82      \\
            81.35       & 0.43833     & 32.70      & 18.65      \\
            81.00       & 0.44182     & 32.75      & 18.50      \\
            79.10       & 0.45867     & 31.39      & 17.66      \\
            \midrule
        \rowcolor[HTML]{DDEBF7} 
            81.36       & 0.43825     & 32.57      & 18.52      \\
        \rowcolor[HTML]{DDEBF7} 
            81.32       & 0.43861     & 32.03      & 18.20      \\
            \bottomrule
        \end{tabular}
        \subcaption{RS}
    \end{minipage}

    \caption{
        The performance of low-bit optimizers in 
        (a) Computer Vision (CV);
        (b) Natural Language Processing (NLP);
        (c) Recommender System (RS).
        \includegraphics[width=8pt]{src/flame1.png}: Training from scratch;
        \includegraphics[width=8pt]{src/flame2.png}: Fine-tuning;
        {\color[HTML]{FF0000}{Crash}}: A training failure is encountered.
    }
    \label{table-overall}
\end{table*}

\begin{table*}[]
    \centering
    \begin{minipage}{0.97\textwidth}
        \centering
        \begin{tabular}{l|r|cccccrc}
            \toprule
            \multicolumn{1}{l}{}                                                                        & \multicolumn{1}{l|}{} & \multicolumn{6}{c}{LLaMA-7B \includegraphics[width=8pt]{src/flame2.png}}                    \\
            \multicolumn{1}{l}{}                                                                       & \multicolumn{1}{l|}{} & MMLU  & ARC-e & ARC-c & OBQA  & SIQA  & \includegraphics[width=6.5pt]{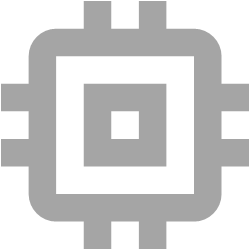}~Memory  & \includegraphics[width=6.5pt]{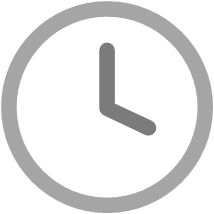}~Time \\
            \midrule
                                                                                    & Pretrained           & 35.34 & 48.85 & 43.39 & 29.40 & 44.58 &        &  \\
        \rowcolor[HTML]{EDEDED} 
                                                                                    & 32-bit AdamW         & 40.80 & 60.61 & 47.01 & 43.80 & 48.50 & 50.24GB & 8.7h \\
                                                                                    & 16-bit AdamW         & 40.94 & 58.73 & 47.23 & 41.40 & 48.21 & 25.12GB & 8.9h \\
                                                                                    & 8-bit AdamW          & 40.34 & 59.49 & 44.52 & 42.20 & 48.05 & 12.97GB & 9.3h \\
                                                                                    & 4-bit AdamW          & 40.58 & 59.20 & 45.76 & 42.41 & 48.94 & 6.69GB & 9.1h \\
                                                                                    & 2-bit AdamW          & 39.84 & 57.26 & 45.09 & 39.53 & 47.31 & 3.41GB & 9.0h \\
            \midrule
        \rowcolor[HTML]{DDEBF7} 
        \multicolumn{1}{c|}{\cellcolor[HTML]{DDEBF7}}                                & 4/2-bit AdamW        & 41.03 & 60.08 & 46.44 & 43.07 & 47.88 & 5.18GB & 9.1h \\
        \rowcolor[HTML]{DDEBF7} 
        \multirow{-2}{*}{\rotatebox{90}{\tiny \textbf{SOLO}}} & 2-bit AdamW          & 40.36 & 59.85 & 45.88 & 40.47 & 47.21 & 3.61GB & 9.1h\\
            \bottomrule
        \end{tabular}
        \subcaption{
            LLM (Large Language Model).
            The Wilcoxon signed-rank test (two-tailed) between the SOLO-quantized 4/2-bit AdamW optimizer 
            and its full 32-bit counterpart indicates a $p$-vlue of 0.125,
            namely no statistically significant difference in performance.
        }
    \end{minipage}

    \vspace{0.7cm}

    \begin{minipage}{0.97\textwidth}
        \centering
        \begin{tabular}{l|r|ccccrc}
            \toprule
            \multicolumn{1}{l}{}                                                              & \multicolumn{1}{l|}{} & \multicolumn{5}{c}{LLaVA-1.5 \includegraphics[width=8pt]{src/flame2.png}}                                        & \multicolumn{1}{l}{} \\
            \multicolumn{1}{l}{}                                                                       & \multicolumn{1}{l|}{} & ScienceQA & TextVQA & POPE     & MME & \includegraphics[width=6.5pt]{src/memory.png}~Memory & \includegraphics[width=6.5pt]{src/time.png}~Time                 \\
            \midrule
        \rowcolor[HTML]{EDEDED} 
                                                                                    & 32-bit AdamW         & 70.81     & 58.50   & 72.40    & 1485.53     &  50.49GB & 22.95h               \\
                                                                                    & 16-bit AdamW         & 71.30     & 58.04   & 72.61    & 1451.10     &  25.25GB   & 22.70h                \\
                                                                                    & 8-bit AdamW          & 70.62     & 57.98   & 72.79    & 1463.13     &  13.05GB   & 25.59h               \\
                                                                                    & 4-bit AdamW          & 70.69     & 58.62   & 73.01    & 1439.57     &  6.76GB   & 24.39h               \\
                                                                                    & 2-bit AdamW          & 28.93     & 54.10   & 72.57    & 1417.93     &  3.61GB   & 23.87h               \\
            \midrule
        \rowcolor[HTML]{DDEBF7} 
        \multicolumn{1}{c|}{\cellcolor[HTML]{DDEBF7}}                                & 4/2-bit AdamW        & 70.81     & 58.10   & 72.43    & 1474.91        & 5.38GB   & 24.53h               \\
        \rowcolor[HTML]{DDEBF7} 
        \multirow{-2}{*}{\rotatebox{90}{\tiny \textbf{SOLO}}} & 2-bit AdamW          & 69.51     & 57.76   & 72.34    & 1472.77    & 3.81GB   & 23.81h               \\
            \bottomrule
        \end{tabular}
        \subcaption{
            LVM (Large Vision Model). 
            The Wilcoxon signed-rank test (two-tailed) between the SOLO-quantized 4/2-bit AdamW optimizer 
            and its full 32-bit counterpart indicates a $p$-vlue of 0.285,
            namely no statistically significant difference in performance.
        }
    \end{minipage}
    \caption{The comparison of low-bit optimizers in giant model fine-tuning.
        (a) LLaMA-7B instruction tuning;
        (b) LLaVA-1.5 visual instruction tuning based on Vicuna-7B-v1.5.
        `Memory': The memory storage required for the optimizer;
        `Time': The total training time (in hours).
    }
    \label{table-overall-giant}
\end{table*}
\clearpage
}

\afterpage{
\begin{figure*}[t]
    \centering
    \begin{minipage}{0.97\textwidth}
        \setlength{\tabcolsep}{5.5pt}
        \centering
        \begin{tabular}{c|ccccc|ccccc}
            \toprule
        \multicolumn{1}{c|}{}  & \multicolumn{5}{c|}{Block size: 128}  & \multicolumn{5}{c}{Block size: 2048} \\
        \multicolumn{1}{c|}{$\quant$} & MMLU  & ARC-e & ARC-c & OBQA  & SIQA  & MMLU  & ARC-e & ARC-c & OBQA  & SIQA  \\
            \midrule
        DE                    & 0.00  & 0.00  & 0.00  & 0.00  & 32.91 & 0.00  & 0.00  & 0.00  & 0.00  & 32.91 \\
        Linear                & 40.82 & 57.73 & 46.67 & 40.60 & 47.39 & 38.46 & 53.56 & 44.07 & 36.07 & 45.34 \\
        \rowcolor[HTML]{DDEBF7} 
        Ours                  & 41.03 & 60.08 & 46.44 & 43.07 & 47.88 & 40.63 & 58.20 & 46.67 & 41.07 & 48.36 \\
            \bottomrule
        \end{tabular}
    \end{minipage}
    \vspace{0.1cm}

    \begin{minipage}{0.47\textwidth}
        \centering
        \includegraphics[width=0.9\linewidth]{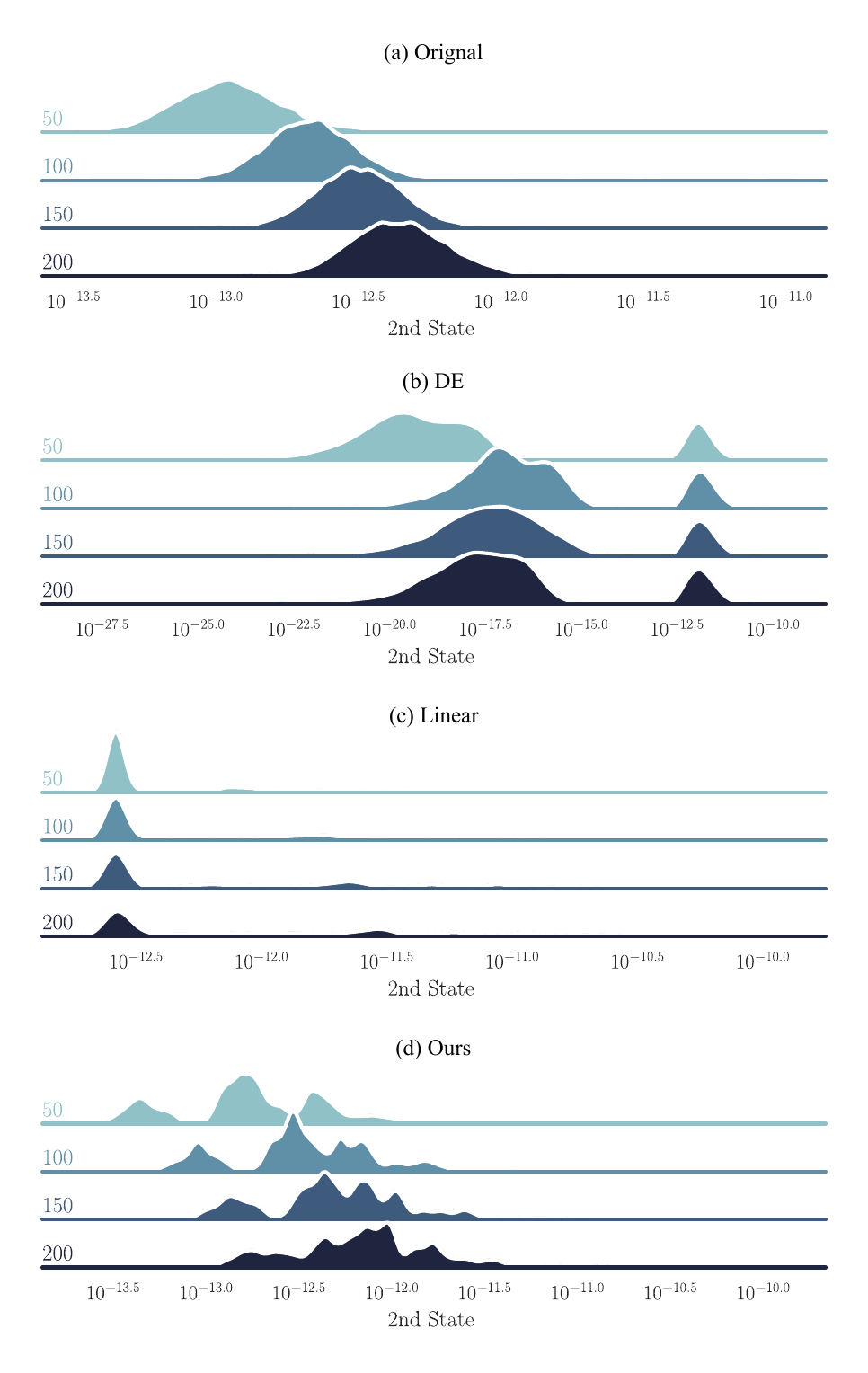}
        \subcaption*{Block size: 128}
    \end{minipage}
    \hfil
    \begin{minipage}{0.47\textwidth}
        \centering
        \includegraphics[width=0.9\linewidth]{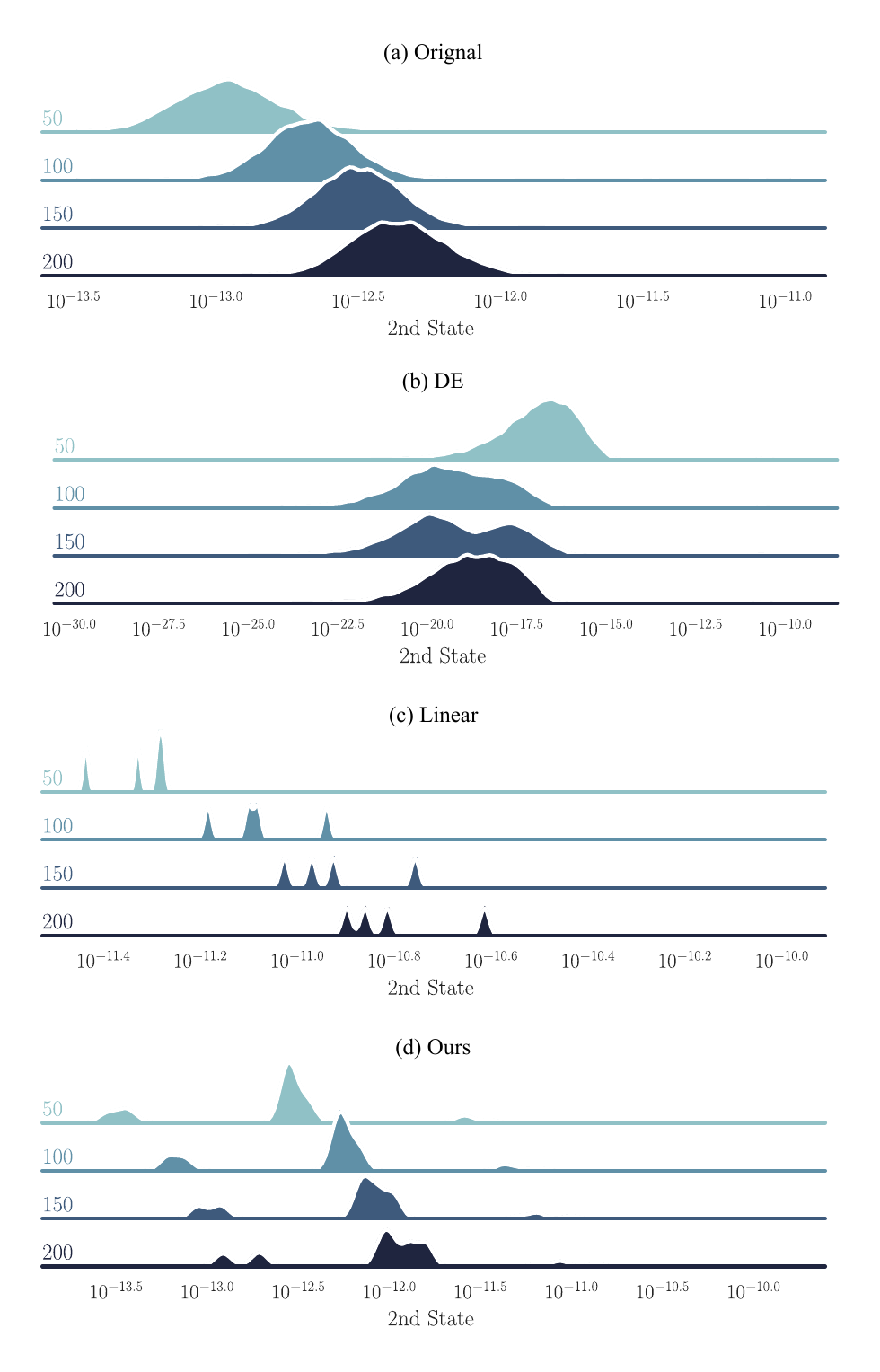}
        \subcaption*{Block size: 2048}
    \end{minipage}
    \caption{
        Block size generalizability.
        The top panel reports the results fine-tuned with a 4/2-bit AdamW whose \second EMA update is quantized by DE, 
        Linear (excluding the zero point), 
        and the tailored logarithmic quantization.
        The bottom panel illustrates how the \second state distribution over partial weights changes at steps 50, 100, 150, and 200.
    }
	\label{fig-unsign-state-visual}
\end{figure*}

\begin{figure*}[t]
	\centering
	\includegraphics[width=0.95\textwidth]{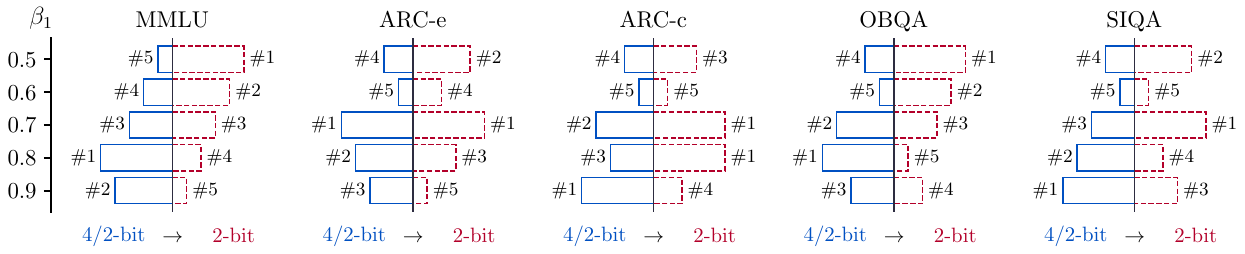}
    \caption{
        The necessity of reducing $\beta_1$ when ultra-low-bit quantization is applied to the \first states.
        `$\#$' indicates the rank (lower is better) among fine-tuned models of varying $\beta_1$ (from 0.5 to 0.9).
    }
	\label{fig-beta-ablation}
\end{figure*}
\clearpage
}

In Table~\ref{table-overall} and \ref{table-overall-giant}, 
we compare SOLO-quantized 4/2-bit Adam(W) and 2-bit Adam(W) with other low-precision counterparts,
covering diverse models, tasks, and training strategies.
There are some promising observations:
\begin{itemize}[leftmargin=*]
    \item 
    Through the comparisons presented in Table~\ref{table-overall}, 
    particularly the results obtained by training from scratch~(\includegraphics[width=8pt]{src/flame1.png}), 
    it is evident that the original full-precision (\ie, 32-bit) states contain considerable redundancy, 
    allowing a straightforward reduction in precision to 8-bit format. 
    However, this redundancy alone is insufficient to support further reductions to 4, 3, or 2 bits without addressing the aforementioned challenges,
    resulting in potential training instability or even training failure.
    Remark that in the original paper~\cite{li20234bit}, 
    the 4-bit Adam(W) optimizer achieves competitive performance on the Swin-T training task. 
    The discrepancy may arise from our adoption of a stricter unified block-wise quantization strategy (tensors containing less than 4096 entries are quantized as well) to ensure a fair comparison.
    Owing to the attention to signal swamping and high variance, 
    SOLO-quantized optimizers are able to achieve comparable performance even in the extreme case of 2-bit quantization.
    \item 
    LLM fine-tuning demonstrates a surprising tolerance to quantization errors: 
    no training collapses are observed, and the resulting performance consistently surpasses that of the original base model.
    However, 8-, 4-, and 2-bit optimizers quantized by previous strategies evidently converge to local optima, 
    making it difficult to believe that they can replace the full-precision AdamW optimizer.
    The proposed SOLO quantization scheme substantially mitigates this issue, 
    owing to its specific considerations of signal swamping and high gradient variance challenges.
    The 4/2-bit variant outperforms most low-precision counterparts, 
    with performance on three out of five benchmarks even exceeding that of the half-precision (\ie, 16-bit) AdamW optimizer.
    Admittedly, fine-tuning with the 2-bit configuration remains highly challenging for large-scale fine-tuning; 
    nonetheless, it still attains performance comparable to that of the 8-bit AdamW optimizer.
    \item
    For visual instruction tuning, the performance degradation caused by quantization errors is negligible until 2-bit representation. 
    This is because the MLP projector that bridges the visual encoder tower (CLIP) and the LLM (Vicuna-7B-v1.5) 
    has already been pretrained in full precision.
    However, scaling to a fully 2-bit setting remains a significant challenge.
    Thanks to SOLO's gradient variance control, the 2-bit quantized AdamW optimizer under SOLO still exhibits competitive performance.
    It is worth noting that, in some cases, the ultra-low-bit optimizer even outperforms its full-precision counterpart. 
    This phenomenon can be attributed to the fact~\cite{zhang2024adammini} 
    that grouping elements within a block to share the same adaptive learning rate 
    (\ie, the \second moment) may facilitate a better convergence.
    \item 
    Note that low-bit optimizers typically incur additional computational overhead, resulting in slightly longer training times. 
    Nonetheless, this overhead could be further reduced by optimizing our implementation with fused operators, as suggested by \citep{li20234bit}.
    However, the implementation of the specified fused operators demands technical engineering effort, 
    which lies beyond the scope of this work and is left for future work.
\end{itemize}

\subsection{Generalizability}

\begin{itemize}[leftmargin=*]
    \item \textbf{Block size.}
    Block-wise quantization has been widely adopted to narrow the numerical range subject to quantization, 
    thereby significantly reducing quantization errors. 
    However, the block size cannot be indefinitely reduced, as each 32-bit scale factor $\Delta$ 
    introduces an additional overhead of $\sfrac{1}{4}$ bits per block. 
    In the case of SOLO, the presence of a buffered logarithmic base incurs the double overhead, 
    amounting to $\sfrac{1}{2}$ bits per block. Consequently, 
    a block size of 128 is approximately the smallest permissible.
    Nonetheless, we argue that the proposed logarithmic quantization has remarkable generalizability in this aspect. 
    \figurename~\ref{fig-unsign-state-visual} compares the results under block sizes of 128 and 2048, 
    where our method significantly outperforms the others. 
    The bottom panel further substantiates the superiority of our approach, 
    as the proposed logarithmic quantization preserves state dynamics that closely resemble those of the original 32-bit AdamW.
    \item \textbf{AdaBelief.}
    Although the results in Table~\ref{table-overall} involve the Adam and AdamW optimizers, 
    it is of interest to examine the generalizability of SOLO to other optimizers that incorporate an EMA update mechanism. 
    Table~\ref{table-other-optimizers} presents the application of SOLO to AdaBelief~\cite{Zhuang2020adabelief}, 
    which also demonstrates comparable performance.
    Remark that SOLO can also be applied to lightweight optimizers such as Adafactor~\cite{shazeer2018adafactor}. 
    We omit the comparison because the use of FSDP (a distributed algorithm that shards weight, gradient, and optimizer memory across different GPUs) 
    renders the comparison infeasible. 
    In this case, the optimizer states are flattened and partitioned, 
    making the row-wise and column-wise computations required by Adafactor invalid.
    \item \textbf{Larger models.}
    We further apply SOLO to larger models, including LLaMA-13B and LLaMA-33B as presented in \tablename~\ref{table-13b-33b}.
    Following the experimental setup for the 13B model as recommended in~\cite{alpaca}, 
    SOLO-quantized 4/2-bit AdamW achieves performance comparable to that of full-precision AdamW.
    In contrast, the 4-bit AdamW counterpart proposed by~\cite{li20234bit} exhibits noticeable performance degradation.
\end{itemize}

\begin{table}[h]
    \centering
    \setlength{\tabcolsep}{3pt}
    \begin{tabular}{c|ccccc}
        \toprule
    \multicolumn{1}{c|}{AdaBelief} & MMLU  & ARC-e & ARC-c & OBQA  & SIQA  \\
        \midrule
    \rowcolor[HTML]{EDEDED} 
    32-bit & 40.78 & 59.55 & 44.63 & 42.73 & 47.80 \\
    \rowcolor[HTML]{DDEBF7} 
    4/2-bit & 40.39 & 58.96 & 45.88 & 42.47 & 47.94 \\
        \bottomrule
    \end{tabular}
    \caption{Application of SOLO to AdaBelief.}
    \label{table-other-optimizers}
\end{table}

\begin{table*}
    \centering
\begin{tabular}{r|ccccc|ccccc}
    \toprule
\multicolumn{1}{l|}{} & \multicolumn{5}{c|}{LLaMA-13B \includegraphics[width=8pt]{src/flame2.png}}         & \multicolumn{5}{c}{LLaMA-33B \includegraphics[width=8pt]{src/flame2.png}}         \\
\multicolumn{1}{c|}{AdamW} & MMLU  & ARC-e & ARC-c & OBQA  & SIQA  & MMLU  & ARC-e & ARC-c & OBQA  & SIQA  \\
    \midrule
\rowcolor[HTML]{EDEDED} 
32-bit & 47.76 & 73.84 & 50.40 & 52.40 & 55.65 & 55.70 & 86.30 & 71.19 & 65.53 & 61.91 \\
4-bit & 47.10 & 73.49 & 51.64 & 51.47 & 55.73 & 55.27 & 85.24 & 70.17 & 64.00 & 61.68 \\
\rowcolor[HTML]{DDEBF7} 
(SOLO) 4/2-bit  & 47.49 & 73.72 & 51.98 & 51.80 & 55.78 & 55.81 & 85.27 & 71.19 & 63.80 & 60.76 \\
    \bottomrule
\end{tabular}
    \caption{Application of SOLO to LLaMA-13B and LLaMA-33B instruction tuning.}
    \label{table-13b-33b}
\end{table*}

\subsection{Empirical Analysis}

\begin{figure}
	\centering
	\includegraphics[width=0.45\textwidth]{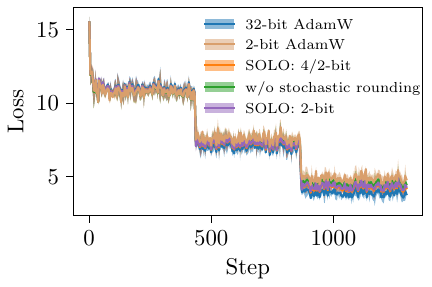}
    \caption{
        The training loss on LLaMA fine-tuning.
    }
	\label{fig-training-loss}
\end{figure}

\begin{figure}
	\centering
	\includegraphics[width=0.45\textwidth]{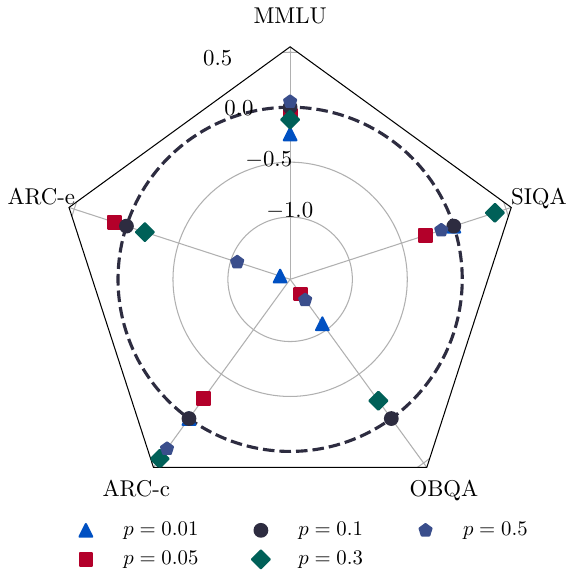}
    \caption{
        Sensitivity analysis of $p$-quantile.
        For clarity, we illustrate the differences relative to the default value of $p=0.1$.
    }
	\label{fig-quantile}
\end{figure}

Some empirical analysis is performed in this part to provide intuitive exposition of SOLO.
\begin{itemize}[leftmargin=*]
    \item \textbf{Modified $\beta_1$.}
    We discuss in Section~\ref{section-signed} the necessity of reducing $\beta_1$ to regulate the gradient variance introduced by quantization. 
    As shown in \figurename~\ref{fig-beta-ablation}, 
    the LLaMA model fine-tuned with 4/2-bit AdamW performs optimally at $\beta_1 = 0.8$ or $0.9$, and its performance deteriorates as $\beta_1$ decreases. Conversely, the performance of the model improves in the 2-bit case and reaches its best near $\beta_1 = 0.5$.
    This confirms our theoretical findings and the rationale summarized in \figurename~\ref{fig-beta}.
    \item \textbf{Loss curve.}
    As illustrated in \figurename~\ref{fig-training-loss}, 
    the SOLO-quantized 4/2-bit and 2-bit AdamW optimizers demonstrate loss curves closely resembling that of the original full-precision AdamW. 
    In contrast, the 2-bit AdamW that incorporates the quantization approaches proposed in~\cite{li20234bit} exhibits significantly poorer convergence, 
    consequently resulting in consistently inferior performance.
    In addition, 
    the specialized rounding mechanism introduced in Eq.~\eqref{eq-unsigned-solo} is essential; 
    otherwise, a deteriorated loss curve would be observed.
    \item \textbf{Logarithmic base.}
    \figurename~\ref{fig-quantile} demonstrates how sensitive of SOLO to the choice of logarithmic base $\alpha$. 
    Recall that $\alpha = (x_{p} / \Delta)^{1/(2^b - 1)}$ is determined by the $p$-quantile $x_p$. 
    When $p$ is too large (\ie, $p \rightarrow 1$), more small state values are directly truncated, leading to high-variance learning rates. 
    Conversely, when $p$ is too small (\ie, $p \rightarrow 0$), the base choice tends to be an outlier, resulting in unacceptable quantization errors.
    Fortunately, SOLO exhibits a satisfactory robustness as $p \in [0.05, 0.3]$.
\end{itemize}

\section{Conclusion}
SOLO is designed to slim down the most memory-intensive component during training, guided by two key principles:
(I) Generalizability.
Although certain optimizers are lightweight by design, 
their generalizability has not been thoroughly validated compared to the widely adopted Adam(W) optimizer.
SOLO is a model-, task-, and hardware-independent technique, 
as it is capable of keeping consistent optimization dynamics even after ultra-low-bit quantization.
(II) Flexibility.  
The application of certain distributed algorithms and network quantization approaches demands substantial engineering effort.
On the contrary,
SOLO can be seamlessly applied to most Adam-style optimizers without the need to retune hyperparameters such as the learning rate or weight decay.

Previous efforts have achieved certain success in 8-bit or 4-bit quantization; however, they violate the generalizability principle.
On the one hand, they overlook the signal swamping problem frequently arised in unsigned EMA updates, 
leading to a poor approximation of the true state dynamics.
In other words, their application essentially changes the nature of the original optimizer in the ultra-low-bit case, 
thereby failing to guarantee that performance is maintained.
On the other hand, signed quantization approaches can not be directly scaled to ultra-low precision owing to the increased gradient variance; otherwise, training failures are likely to occur.
SOLO addresses these two challenges through comprehensive theoretical analysis, adhering to the principles of both generalizability and flexibility.

The limitations of SOLO motivate future research on 
how to further unleash the potential of ultra-low-bit signed quantization in large-scale model training.
According to the analysis in this paper,
the signed EMA update is highly susceptible to the quantization errors 
due to the additional burden of sign representation 
and its direct impact on the update direction.
Recent advances in sign gradient descent~\cite{chen2024lion} may offer particular advantages in addressing this issue.
Finally, although the superiority of SOLO has been demonstrated across various domains and training protocols, 
its effectiveness in large-scale model (\eg, 175B) pretraining remains unverified due to prohibitive computational costs. 
Conducting such an extensive evaluation would require the resources of a large AI research organization, which is beyond the scope of this study.

The prosperity of AI stems from the collective efforts of individual researchers within the community; 
however, the rapid scaling of model sizes is rendering fundamental research 
accessible only to those equipped with substantial computational infrastructure.
There is an urgent need for affordable alternatives to reinvigorate research interest.
SOLO advances this goal by slimming down the optimizer, which is the most memory-intensive part of the training process.

\bibliography{refs}
\bibliographystyle{icml2025}


\newpage
\appendix
\onecolumn

\section{Quantization Methods}

In this part, 
we provide a detailed introduction to the aforementioned Linear and Dynamic Exponent (DE) quantization techniques. 
Particular emphasis is placed on the following mapping formulation, 
as it effectively unifies the majority of unsigned and signed quantization methods.
\begin{align}
    q = \quant(x) := \argmin_{k=0}^{2^b - 1} \big|\frac{x}{\Delta} - \level_k \big|.
\end{align}

\begin{figure}
	\centering
	\includegraphics[width=0.9\textwidth]{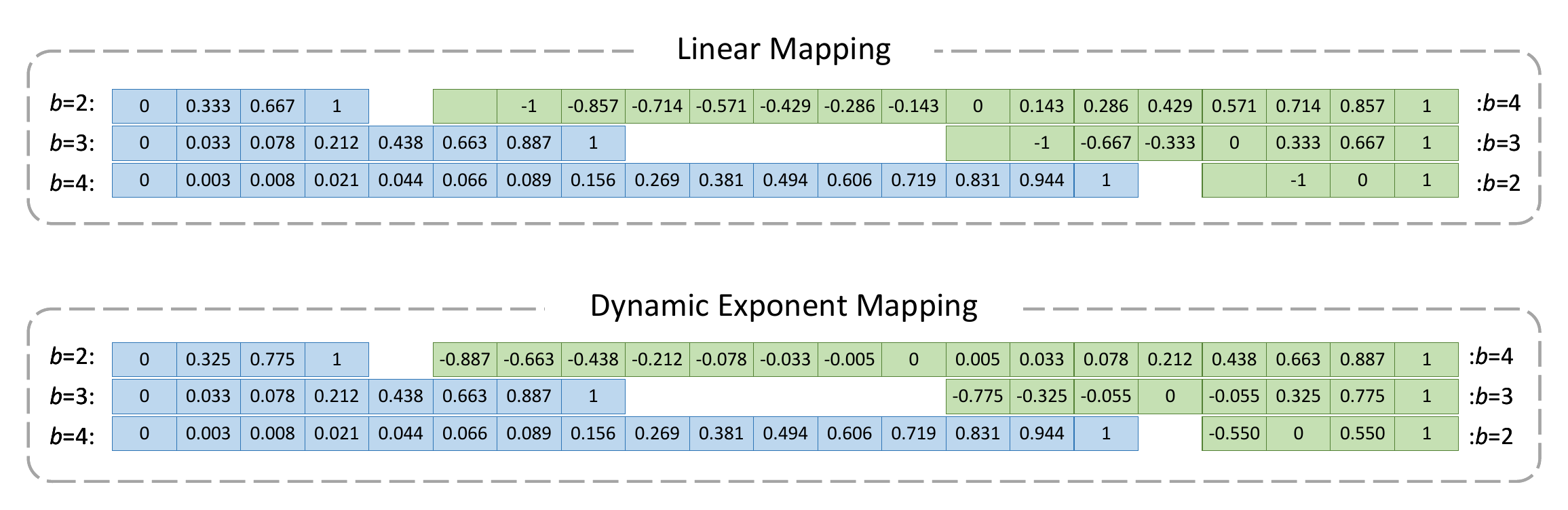}
	\caption{
        Linear and DE mappings for \textcolor{unsignedcolor}{unsigned} and \textcolor{signedcolor}{signed} cases.
	}
	\label{fig-mapping}
\end{figure}

\textbf{Linear mapping} represents one of the most widely adopted quantization methods 
due to its straightforward and practical implementation with levels evenly distributed over [0, 1] and [-1, 1]:
\begin{align}
    \text{Unsigned:} \quad \{0, \frac{1}{2^{b} - 1}, \ldots, 1\};
    \quad \quad
    \text{Signed:} \quad \{0, \pm \frac{1}{2^{b-1} - 1}, \ldots, \pm 1\}.
\end{align}
The 4-bit Adam developed by \citet{li20234bit} observed a zero-point problem when applying the unsigned linear quantization to \second state update,
yielding a slightly different variant of $\{\frac{1}{2^{b}}, \ldots, 1\}$.

\textbf{Dynamic exponent mapping~\cite{dettmers2015de}} places greater emphasis on handling both small and large magnitude values,
as the received signals generally exhibit a non-uniform distribution.
In the unsigned case, the $b$ bits are divided to serve three distinct functions:
\begin{equation}
    b = \underbrace{E}_{\text{exponent bits}} + \underbrace{1}_{\text{indicator bit}} + \underbrace{F}_{\text{fraction bits}}.
\end{equation}
The leading zero bits $E$ represents the magnitude of the exponent.
The subsequent indicator bit is the first bit set to one, signifying that all remaining bits are allocated for fraction representation. 
Then, a $F$-bit linear quantization is performed over the range [0.1, 1].
Naturally, the first bit has to be reserved for the signed case.
The 8-bit Adam developed by \citet{dettmers20228bit} simplifies this quantization by removing the indicator bit.
We illustrate the specific linear and DE mappings across $b=2,3,4$ bits in \figurename~\ref{fig-mapping}.

Denoted by $r_{\min}, r_{\text{median}}, r_{\max}$ the minimum, median, and maximum over all radii,
it is easy to obtain the following fact:
\begin{fact}
    \label{fact-radii}
    The maximum radius:
    \begin{itemize}
        \item Linear:
        \begin{align*}
            \text{Unsigned:} \quad & r_{\min} = r_{\text{median}} = r_{\max} \approx \underbrace{0.002}_{\text{8-bit}},\underbrace{0.033}_{\text{4-bit}},\underbrace{0.071}_{\text{3-bit}},\underbrace{0.167}_{\text{2-bit}}; \\
            \text{Signed:} \quad & r_{\min} = r_{\text{median}} = r_{\max} \approx \underbrace{0.004}_{\text{8-bit}},\underbrace{0.071}_{\text{4-bit}},\underbrace{0.167}_{\text{3-bit}},\underbrace{0.500}_{\text{2-bit}}.
        \end{align*}
        \item DE:
        \begin{align*}
            \text{Unsigned:} \quad & 
            r_{\min} \approx \underbrace{0.000}_{\text{8-bit}},\underbrace{0.002}_{\text{4-bit}},\underbrace{0.016}_{\text{3-bit}},\underbrace{0.113}_{\text{2-bit}}, \:
            r_{\text{median}} \approx \underbrace{0.002}_{\text{8-bit}},\underbrace{0.034}_{\text{4-bit}},\underbrace{0.067}_{\text{3-bit}},\underbrace{0.163}_{\text{2-bit}}, \\
            & r_{\max} \approx \underbrace{0.004}_{\text{8-bit}},
            \underbrace{0.056}_{\text{4-bit}},
            \underbrace{0.113}_{\text{3-bit}},\underbrace{0.225}_{\text{2-bit}}; \\
            \text{Signed:} \quad & 
            r_{\min} \approx \underbrace{0.000}_{\text{8-bit}},\underbrace{0.003}_{\text{4-bit}},\underbrace{0.028}_{\text{3-bit}},\underbrace{0.225}_{\text{2-bit}}, \:
            r_{\text{median}} \approx 
            \underbrace{0.004}_{\text{8-bit}},
            \underbrace{0.008}_{\text{7-bit}},
            \underbrace{0.017}_{\text{6-bit}},
            \underbrace{0.034}_{\text{5-bit}},
            \underbrace{0.067}_{\text{4-bit}},\underbrace{0.135}_{\text{3-bit}},\underbrace{0.275}_{\text{2-bit}}, \\
            & r_{\max} \approx \underbrace{0.007}_{\text{8-bit}},
            \underbrace{0.014}_{\text{7-bit}},
            \underbrace{0.028}_{\text{6-bit}},
            \underbrace{0.056}_{\text{5-bit}},
            \underbrace{0.113}_{\text{4-bit}},
            \underbrace{0.225}_{\text{3-bit}},\underbrace{0.275}_{\text{2-bit}}.
        \end{align*}
    \end{itemize}
\end{fact}

\section{Proofs}
\label{appendix-proofs}

\subsection{Proofs of Signal Swamping}

\begin{proof}[Proof of Theorem~\ref{thm-dsp}]
Notice that the quantization at $t+1$ update can be formulated as follows
\begin{align*}
    q_{t+1} 
    &= \quant (\hat{x}_{t+1}) \\
    &= \quant (\beta \cdot \tilde{x}_t + (1 - \beta) \cdot z_{t+1}) \\
    &= \quant (\beta \cdot \level_{q_t} \cdot \Delta_t + (1 - \beta) \cdot z_{t+1}) \\
    &= \argmin_{k=0}^{2^b - 1}
    \big|\beta \cdot \level_{q_t} \cdot \frac{\Delta_t}{\Delta_{t+1}} + (1 - \beta) \cdot \frac{z_{t+1}}{\Delta_{t+1}} - \level_k \big|.
\end{align*}
It can be easily derived that $q_{t+1} = q_t$ if
\begin{equation}
    \label{eq-swamping-condition}
    \big|\beta \cdot \level_{q_t} \cdot \frac{\Delta_t}{\Delta_{t+1}} + (1 - \beta) \cdot \frac{z_{t+1}}{\Delta_{t+1}} - \level_{q_t} \big| \le r.
\end{equation}
Remark that this condition is also necessary when $|\level_{q_t} - \level_{q_t - 1}| = |\level_{q_t + 1} - \level_{q_t}|$, \eg, in the case of linear quantization.

Reranging the left hand side gives
\begin{align*}
    & \big|\beta \cdot \level_{q_t} \cdot \frac{\Delta_t}{\Delta_{t+1}} + (1 - \beta) \cdot \frac{z_{t+1}}{\Delta_{t+1}} - \level_{q_t} \big| \\
    =& 
    \big|
    \beta (\frac{\Delta_t}{\Delta_{t+1}} - 1) \level_{q_t} + (1 - \beta) (\frac{z_{t+1}}{\Delta_{t+1}} - \level_{q_t})
    \big| \\
    \le &
    \Big|
    \frac{\Delta_t}{\Delta_{t+1}} - 1
    \Big|
    + (1 - \beta) \cdot \Big|\frac{z_{t+1}}{\Delta_{t+1}} - \level_{q_t} \Big|.
\end{align*}
Therefore, the inequality \eqref{eq-swamping-condition} holds true as long as
\begin{equation*}
    \Big|
    \frac{\Delta_t}{\Delta_{t+1}} - 1
    \Big|
    +(1 - \beta) \cdot \Big|\frac{z_{t+1}}{\Delta_{t+1}} - \level_{q_t} \Big| \le r.
\end{equation*}

\end{proof}

\begin{corollary}[A generalized version of Corollary~\ref{corollary-linear-de}]
    If $\Delta_{t+1} = \Delta_t$ and $r_t \ge 2 (1 - \beta)$, 
    the unsigned quantized state $q_t$ remains unchanged for any $|z_{t+1}| \le \Delta_t$.
    To be specific,
    \begin{itemize}[leftmargin=*]
        \item 
        For linear quantization, 
        the received signals are swamped by states corresponding to any level when
        \begin{align}
            \text{Unsigned:} \:
            \beta \gtrapprox 
            \underbrace{0.999}_{\text{8-bit}}, 
            \underbrace{0.967}_{\text{4-bit}}, 
            \underbrace{0.929}_{\text{3-bit}}, 
            \underbrace{0.833}_{\text{2-bit}};
            \quad
            \text{Signed:} \:
            \beta \gtrapprox 
            \underbrace{0.998}_{\text{8-bit}}, 
            \underbrace{0.965}_{\text{4-bit}}, 
            \underbrace{0.917}_{\text{3-bit}}, 
            \underbrace{0.750}_{\text{2-bit}}.
        \end{align}
        \item 
        For dynamic exponent quantization, 
        the received signals are swamped by states corresponding to half of the levels when
        \begin{align}
            \text{Unsigned:} \:
            \beta \gtrapprox
            \underbrace{0.998}_{\text{8-bit}}, 
            \underbrace{0.966}_{\text{4-bit}}, 
            \underbrace{0.933}_{\text{3-bit}}, 
            \underbrace{0.837}_{\text{2-bit}};
            \quad
            \text{Signed:} \:
            \beta \gtrapprox 
            \underbrace{0.998}_{\text{8-bit}}, 
            \underbrace{0.967}_{\text{4-bit}}, 
            \underbrace{0.933}_{\text{3-bit}}, 
            \underbrace{0.863}_{\text{2-bit}}.
        \end{align}
    \end{itemize}
\end{corollary}

\begin{proof}
    According to the proof of Theorem~\ref{thm-dsp}, we know the quantized state becomes swamped if
    \begin{equation}
        (1 - \beta) \cdot \Big|\frac{z_{t+1}}{\Delta_{t+1}} - \level_{q_t} \Big| \le r.
    \end{equation}
    It provides two sufficient conditions for the unsigned and signed cases, respectively; that is,
    \begin{equation}
            \beta \ge 
            \left \{
                \begin{array}{ll}
                    1 - r_{\min} & \text{if } unsigned, \\
                    1 - r_{\min} / 2 & \text{if } signed.
                \end{array}
            \right .
    \end{equation}
The results then follow from the Fact~\ref{fact-radii}.
\end{proof}

\subsection{Proof of Proposition~\ref{prop-unsigned-variance}}

We first present a trivial lemma regarding a two-point distribution.

\begin{lemma}
    \label{lemma-two-point}
    Let $X$ be a random variable with
    \begin{equation}
        \mathbb{P}[X=a] = p, \quad 
        \mathbb{P}[X=b] = q = 1 - p.
    \end{equation}
    Then, we have
    \begin{equation}
        \mathbb{E}[X] = pa + qb, \quad
        \mathbf{Var}[X] = pq (a - b)^2.
    \end{equation}
\end{lemma}

\begin{proof}
    The proof is obvious.
\end{proof}

\begin{proof}[Proof of Proposition~\ref{prop-unsigned-variance}]

For a $\level_{k-1} \le x / \Delta \le \level_k$, we have
\begin{equation*}
    \frac{1}{\sqrt{\hat{x}}} =
    \left \{
        \begin{array}{ll}
            \frac{\Delta}{\sqrt{\level_{k-1}}},  & w.p. \quad \frac{\level_k - x / \Delta}{ \level_k - \level_{k-1}}, \\
            \frac{\Delta}{\sqrt{\level_{k}}}, & w.p. \quad \frac{x / \Delta - \level_{k-1}}{ \level_k - \level_{k-1}}.
        \end{array}
    \right .
\end{equation*}
Applying Lemma~\ref{lemma-two-point} with $p = \frac{\level_k - x / \Delta}{ \level_k - \level_{k-1}}, a = \frac{\Delta}{\sqrt{\level_{k-1}}}, b = \frac{\Delta}{\sqrt{\level_{k}}}$, we have
\begin{align*}
    \mathbf{Var}(\frac{1}{\sqrt{\hat{x}}}) = 
    \frac{(\level_k - x / \Delta) (x / \Delta - \level_{k-1}) }{(\level_k - \level_{k-1})^2}
    \cdot 
    \Big ( 
        \frac{\Delta}{\sqrt{\level_{k-1}}} - 
        \frac{\Delta}{\sqrt{\level_{k}}}
    \Big )^2.
\end{align*}
The proof is finished by noting the fact that
the magnitude of the first term is determined by the relative distance $x / \Delta$ within the interval $[\level_{k-1}, \level_{k}]$.
\end{proof}

\subsection{Proof of Proposition~\ref{prop-state-decay}}

For an intermediate state value of $q_t = (\beta^c)^{k + k'}$, 
the state transition probability after receiving a zero-value signal ($z_{t+1} = 0$) can be represented as follows:
\begin{equation}
    q_{t+1} = 
    \left \{
        \begin{array}{ll}
           (\beta^c)^{k + k'}, & w.p. \quad \frac{c - 1}{c}, \\
           (\beta^c)^{k + k' + 1}, & w.p. \quad \frac{1}{c}.\\
        \end{array}
    \right .
\end{equation}
Thus, transitioning from the state value of $(\beta^c)^k$ to $(\beta^c)^{k+s}$ needs exact $s$ independent successes.
Denoted by $N$ the counting of updates required to achieve exact $s$ successes, 
this random variable $N$ adheres to a negative binomial distribution with a success probability of $1/c$.
It has been proved in the majority of standard statistical textbooks that
\begin{equation}
    \mathbb{E}\big[(\beta^c)^k 
    \xrightarrow{\text{Ours}}
    (\beta^c)^{k+s} \big| z \equiv 0
    \big] = 
    \mathbb{E}[N] = \frac{s}{1 / c} = c \cdot s.
\end{equation}

\subsection{Proof of Theorem~\ref{thm-first-variance}}

Revisting the quantized EMA update, we have
\begin{align*}
    \hat{x}_{t+1}
    &= \beta \cdot \tilde{x}_t + (1 - \beta) \cdot g_{t+1} \\
    &= \beta \cdot (\hat{x}_t + \tilde{x}_t - \hat{x}_t) + (1 - \beta) \cdot g_{t+1} \\
    &= \beta \cdot \hat{x}_t + (1 - \beta)  \cdot \Big(\underbrace{g_{t+1} + \frac{\beta}{1 - \beta} (\tilde{x}_t - \hat{x}_t)}_{=: \tilde{g}_{t+1}}  \Big).
\end{align*}
Since standard stochastic rounding is performed independently in each update,
the quantization error can be treated as noise added to the gradient.
Let $x_{t} := \hat{x}_t$, yielding a standard EMA update:
\begin{align*}
    x_{t+1} \leftarrow \beta \cdot x_t + (1 - \beta) \cdot \tilde{g}_{t+1}.
\end{align*}
$\mathbb{E}[\tilde{g}_{t+1}] = \mathbb{E}[g_{t+1}]$ follows from the zero expectation of quantization noise 
(if the standard stochastic rounding is performed~\cite{xia2020sr}),
and we are to establish its variance next, or equivalently
\begin{equation}
    \mathbf{Var}\Big[\frac{\beta}{1 - \beta} (\tilde{x}_t - \hat{x}_t)\Big]
    =\Big(\frac{\beta}{1 - \beta} \Delta_t \Big)^2 \mathbf{Var}[\level_{q_t} - \frac{\hat{x}_t}{\Delta_t}] \\
    =\Big(\frac{\beta}{1 - \beta} \Delta_t \Big)^2 \mathbf{Var}[\level_{q_t}].
\end{equation}
Without loss of generality, suppose $\level_{k-1} \le \hat{x}_t / \Delta_t \le \level_k$.
It is easy to show that
\begin{equation}
    \mathbf{Var}[\level_{q_t}] = 
    \frac{(\level_k - x / \Delta) (x / \Delta - \level_{k-1}) }{(\level_k - \level_{k-1})^2}
    \cdot 
    ( 
        \level_{k-1} - \level_{k}
    )^2
\end{equation}
by applying Lemma~\ref{lemma-two-point} and noticing the fact that
\begin{equation*}
    \level_{q_t} =
    \left \{
        \begin{array}{ll}
            \level_{k-1},  & w.p. \quad \frac{\level_k - \hat{x}_t / \Delta_t}{ \level_k - \level_{k-1}}, \\
            \level_{k}, & w.p. \quad \frac{\hat{x}_t / \Delta_t - \level_{k-1}}{ \level_k - \level_{k-1}}.
        \end{array}
    \right .
\end{equation*}
Then, we have
\begin{align}
    \mathbf{Var}\Big[\frac{\beta}{1 - \beta} (\tilde{x}_t - \hat{x}_t)\Big]
    &=
    \Big(\frac{\beta}{1 - \beta} \Delta_t \Big)^2 \cdot 
    \frac{(\level_k - x / \Delta) (x / \Delta - \level_{k-1}) }{(\level_k - \level_{k-1})^2} \cdot
    ( 
        \level_{k-1} - \level_{k}
    )^2 \\
    &\le
    \Big(\frac{\beta}{1 - \beta} \Delta_t \Big)^2 \cdot 
    \frac{1}{4} \cdot
    ( 
        \level_{k-1} - \level_{k}
    )^2 \\
    &\lessapprox
    \Big(\frac{\beta}{1 - \beta} \Delta_t \Big)^2 \cdot 
    \frac{1}{4} \cdot (2 r_{\max})^2
    = \Big(\frac{\beta}{1 - \beta}  r_{\max} \Delta_t \Big)^2,
\end{align}
where the first inequality is because
\begin{equation*}
    pq \le \frac{1}{4} \quad \text{if } 0 \le p, q \le 1 \text{ and } p + q = 1,
\end{equation*}
and the second inequality holds true is due to the definition of $r_{\max}$.

\section{Experimental Details}

\subsection{Settings and Hyperparameters}\label{section-settings}

\begin{table}[h]
    \centering
    \begin{tabular}{l|c}
    \toprule
    Hyperparameter & Value \\
    \midrule
    LR           & 5e-4  \\
    Warmup LR & 1.e-7 \\
    LR Scheduler & Inverse sqrt  \\
    Weight Decay & 0     \\
    $\beta_2$        & 0.98 \\
    Max Steps       & 200,000     \\
    Warmup Steps & 4,000  \\
    Label Smoothing       & 0.1     \\
    Dropout       & 0.3     \\
    \bottomrule
    \end{tabular}
    \caption{
        The hyperparameters of Transformer-base training on WMT'14 English-German translation.
    }
    \label{hp-nlp-training}
\end{table}

\begin{table}[h]
    \centering
    \begin{tabular}{l|cccccccc}
    \toprule
    & MNLI    & QNLI   & QQP     & RTE   & MRPC  & SST-2  & COLA  & STS-B \\
    \midrule
    Batch Size    & 32      & 32     & 32      & 16    & 16    & 32     & 16    & 16    \\
    LR            & 1e-5    & 1e-5   & 1e-5    & 2e-5  & 1e-5  & 1e-5   & 1e-5  & 2e-5  \\
    Weight Decay  & 0.1     & 0.1    & 0.1     & 0.1   & 0.1   & 0.1    & 0.1   & 0.1   \\
    $\beta_2$     & 0.98    & 0.98   & 0.98    & 0.98  & 0.98  & 0.98   & 0.98  & 0.98  \\
    Warmup Steps  & 7,432   & 1,986  & 28,318  & 122   & 137   & 1,256  & 320   & 214   \\
    Max Steps     & 123,873 & 33,112 & 113,272 & 2,036 & 2,296 & 20,935 & 5,336 & 3,598  \\
    Max Seq. Len. & 128     & 128    & 128     & 512   & 512   & 512    & 512   & 512   \\
    \bottomrule
    \end{tabular}
    \caption{
        The hyperparameters of RoBERTa-L fine-tuning on GLUE.
    }
    \label{hp-nlp-fine-tuning}
\end{table}

\textbf{Transformer-Base (\includegraphics[width=8pt]{src/flame1.png}) on WMT'14.}
We consider the Neural Machine Translation (NMT) task follow the paper~\cite{ott2018scaling} 
and the codebase\footnote{
    \url{https://github.com/NVIDIA/DeepLearningExamples/blob/master/PyTorch/Translation/Transformer}
}.
We train a Transformer-Base model on the WMT'14 English-German translation task from scratch.
The specific setting can be found in Table~\ref{hp-nlp-training}.
Four A800 GPUs have been used for the runs of tasks.
We report the BLEU~\cite{papineni2002bleu} results (averaged over 3 independent runs) from the best checkpoint.

\textbf{RoBERTa-Large (\includegraphics[width=8pt]{src/flame2.png}) on GLUE.}
NLP fine-tuning  plays a critical role in achieving satisfactory performance across specific scenarios. 
During the early stages of NLP development, 
downstream tasks could not be effectively addressed using a single model, 
thus giving rise to various specialized domains, 
including Natural Language Understanding (NLU), Question Answering (QA), and Natural Language Generation (NLG). 
This study evaluates the performance of a low-bit optimizer in NLU\footnote{
\url{https://github.com/huggingface/transformers/tree/main/examples/pytorch/text-classification}
}
by fine-tuning the RoBERTa-L model~\cite{liu2019roberta} 
on the widely utilized GLUE benchmark~\cite{wang2018glue},
which supports 8 different tasks. 
Their respective experimental settings are listed in \tablename~\ref{hp-nlp-fine-tuning}.
Besides, a linear schedule is employed for learning rate decay.
With the exception of SST-2, which necessitates two GPUs for execution, 
the results on the remaining datasets are obtained using a single RTX 3090 GPU.

\begin{table}[h]
    \centering
    \begin{tabular}{l | c}
    \toprule
    Hyperparameter & Value\\
    \midrule

        Batch Size                  & 512  \\
        LR                          & 5e-4   \\
        Weight Decay                & 0.05     \\
        Emb. Decay                  & 0   \\
        $\beta_2$                       & 0.999  \\
        Epochs                      & 300    \\
    \bottomrule
    \end{tabular}
    \caption{
        The hyperparameters of Swin-T pretraining on ImageNet-1K.
    }
    \label{hp-swin-pretrain}
\end{table}

\textbf{Swin-T (\includegraphics[width=8pt]{src/flame1.png}) on ImageNet-1K.} 
CV pretraining can provide a robust foundational model for a wide range of downstream tasks, 
among which image classification remains the most prevalent task for universal representation learning.
To demonstrate the effectiveness of our approach in computer vision applications, 
we train the tiny variant of the Swin Transformer~\cite{liu2021swin} from scratch on the ImageNet-1K dataset~\cite{deng2009imagenet} for the image classification task\footnote{
    \url{https://github.com/microsoft/Swin-Transformer}
}. 
Detailed experimental configurations are provided in \tablename~\ref{hp-swin-pretrain}. 
The training is conducted using four NVIDIA V100 GPUs. 
For evaluation, we report both top-1 and top-5 accuracy.

\begin{table}[h]
    \centering
    \begin{minipage}{0.3\textwidth}
        \centering
        \begin{tabular}{l|c}
            \toprule
        Hyperparameter & Value \\
            \midrule
        Batch Size                  & 4,096  \\
        LR                          & 1e-3   \\
        Weight Decay                & 0      \\
        Emb. Decay                  & 1e-5   \\
        $\beta_2$                       & 0.999  \\
        Epochs                      & 100    \\
            \bottomrule
        \end{tabular}
        \subcaption{DCN training on Criteo}
    \end{minipage}
    \hfil
    \begin{minipage}{0.3\textwidth}
        \centering
        \begin{tabular}{l|c}
            \toprule
        Hyperparameter & Value \\
            \midrule
        Batch Size                  & 128  \\
        Num Negatives                      & 128    \\
        LR                          & 1e-3   \\
        Weight Decay                & 0      \\
        $\beta_2$                       & 0.999  \\
        Epochs                      & 100    \\
            \bottomrule
        \end{tabular}
        \subcaption{HSTU training on MovieLens}
    \end{minipage}
    \caption{
        The hyperparameters of RS training.
    }
    \label{hp-recsys}
\end{table}

\textbf{DCN (\includegraphics[width=8pt]{src/flame1.png}) on Criteo and HSTU-Large (\includegraphics[width=8pt]{src/flame1.png}) on MovieLens.}
RS training has garnered greater interest from the industrial community compared to NLP and CV scenarios.
Moreover, recommender systems are well-known for their large-scale user and item entities, 
resulting in massive embedding tables. 
This makes them particularly valuable for study in this context, 
as their `gigantic' nature arises not merely from deeper or wider model architectures, but from the scale of the entities themselves.
We consider two primary tasks across the recommendation pipeline. 
Sequential recommendation focuses on retrieving a list of candidate items to be passed to the subsequent ranking stage.
Click-Through Rate (CTR) prediction is a classical industrial application that estimates the probability of a candidate item being clicked. 
For the former, 
we adopt the recently proposed HSTU\footnote{
    \url{https://github.com/facebookresearch/generative-recommenders}
}
model~\cite{zhai2014hstu} on the widely used MovieLens-1M dataset, 
while for the latter, we consider the classic DCN\footnote{
    \url{https://github.com/MTandHJ/RecBoard/tree/master/DCN}
}
model~\cite{wang2017dcn} on the Criteo dataset.
The hyperparameters can be found in \tablename~\ref{hp-recsys}.

\begin{table}[h]
    \centering
    \begin{tabular}{l|ccc}
    \toprule
    Hyperparameter & LLaMA-7B & LLaMA-13B & LLaMA-33B \\
    \midrule
    Batch   Size & 120  & 120 & 128 \\
    LR           & 2e-5  & 1e-5 & 1e-5 \\
    Weight Decay & 0   & 0 & 0  \\
    $\beta_2$        & 0.999 & 0.999 & 0.999 \\
    Warmup Ratio & 0.03 & 0.03 & 0.03  \\
    Epochs       & 3    & 5 & 5 \\
    \bottomrule
    \end{tabular}
    \caption{
        The hyperparameters of LLaMA fine-tuning on Alpaca.
    }
    \label{hp-llm-fine-tuning}
\end{table}

\textbf{LLaMA-7B/13B/33B (\includegraphics[width=8pt]{src/flame2.png}) on Alpaca.}
LLM fine-tuning is one of the most compelling applications, but still faces challenges due to its high memory demands.
Since ChatGPT's pioneering days, a number of open source LLMs~\cite{touvron2023llama,genai2023llama,bai2023qwen,yang2024qwen2,liu2024deepseekv2,liu2024deepseekv3} have been released, with the LLaMA series receiving the most attention.
Hence,
we fine-tune LLaMA-7B/13B/33B~\cite{touvron2023llama} based on Alpaca~\cite{alpaca} codebase\footnote{
    \url{https://github.com/tatsu-lab/stanford_alpaca}
} with the hyperparameters given in \tablename~\ref{hp-llm-fine-tuning}.
Due to the utilization of three H800 GPUs for training 7B and 13B models, 
we adopt a total batch size of 120 instead of the standard 128, with 40 assigned to each GPU.
Besides, we conduct experiments on LLaMA-33B using eight H20 GPUs.
For evaluation on the MMLU~\cite{hendrycks2020mmlu} and standard common sense reasoning benchmarks: ARC easy and challenge~\cite{clark2018arc}, OpenBookQA~\cite{mihaylov2018obqa},
and Social Interaction QA~\cite{sap2019siqa},
we utlize OpenCompass~\cite{2023opencompass}.

\begin{table}[h]
    \centering
    \begin{tabular}{l|c}
    \toprule
    Hyperparameter & Value \\
    \midrule
    Batch   Size & 120   \\
    LR           & 2e-5  \\
    Weight Decay & 0     \\
    $\beta_2$        & 0.999 \\
    Warmup Ratio & 0.03  \\
    Epochs       & 1     \\
    Visual Tower       & clip-vit-large-patch14-336    \\
    Projector Type       & mlp2x\_gelu \\
    \bottomrule
    \end{tabular}
    \caption{
        The hyperparameters of LLaVA-1.5 fine-tuning on a mixture instruction tuning data.
    }
    \label{hp-lvm-fine-tuning}
\end{table}

\textbf{LLaVA-1.5 (\includegraphics[width=8pt]{src/flame2.png}) visual instruction tuning.}
We follow the LLaVA-1.5 codebase\footnote{
    \url{https://github.com/haotian-liu/LLaVA}
} to fine-tune the Vicuna-7B-v1.5 model~\cite{vicuna2023} on a mixed instruction tuning dataset 
comprising COCO~\cite{caesar2018coco}, GQA~\cite{hudson2019gqa}, OCR-VQA~\cite{mishra2019ocr}, TextVQA~\cite{singh2019textvqa}, and VisualGenome~\cite{krishna2017visual}. 
Due to the use of three H800 GPUs for training, we employ a total batch size of 120 instead of the standard 128 by assigning a batch size of 40 to each GPU. 
All other hyperparameters adhere to the recommended configurations and have been detailed in Table~\ref{hp-lvm-fine-tuning}. 
In accordance with the official evaluation protocols, 
we assess the fine-tuned models on several benchmarks, including ScienceQA~\cite{lu2022sqa}, TextVQA~\cite{singh2019textvqa}, POPE~\cite{li2023pope}, and MME~\cite{fu2023mme}.
Specifically, 
we report accuracy for ScienceQA and TextVQA; 
the F1-score averaged across the `Random', `Adversarial', and `Popular' subtasks for POPE; 
the `Perception' score for MME.

\section{Additional Experimental Results}

\subsection{Sensitivity to $p$-quantile}

\begin{table}[]
    \centering
\begin{tabular}{cccccc}
    \toprule
$p$ & MMLU  & ARC-e & ARC-c & OBQA  & SIQA  \\
    \midrule
0        & 40.73 & 59.20 & 45.65 & 42.13 & 48.11 \\
0.01     & 40.78 & 58.61 & 46.44 & 42.00 & 47.88 \\
0.05     & 40.96 & 60.20 & 46.21 & 41.67 & 47.61 \\
0.1      & 41.03 & 60.08 & 46.44 & 43.07 & 47.88 \\
0.3      & 40.91 & 59.91 & 46.89 & 42.87 & 48.28 \\
0.5      & 41.08 & 59.02 & 46.78 & 41.73 & 47.76 \\
0.8      & 40.73 & 57.26 & 46.44 & 39.87 & 47.22 \\
    \bottomrule
\end{tabular}
    \caption{
        Sensitivity analysis of $p$-quantile.
    }
    \label{table-sensitivity}
\end{table}

We present detailed results across various $p$-quantiles (\tablename~\ref{table-sensitivity}) to demonstrate the robustness of SOLO to quantile selection.
SOLO exhibits satisfactory robustness as long as $p \in [0.05, 0.3]$.

\subsection{Momentum Adjustment}

We demonstrate the necessity of reducing $\beta_1$ to regulate the gradient variance in detail.
As can be seen in Table~\ref{table-momentum-adjustment}, 
the LLaMA model fine-tuned with 4-bit signed quantization achieves optimal performance around $\beta_1 = 0.8$, 
and its performance significantly deteriorates as $\beta_1$ decreases. 
Conversely, the performance improves in the 2-bit case and reaches its peak near $\beta_1 = 0.5$.

\begin{table}[]
    \centering
\begin{tabular}{ccccccc}
    \toprule
\multicolumn{1}{l}{}     & $\beta_1$ & MMLU  & ARC-e & ARC-c & OBQA  & SIQA  \\
    \midrule
\multirow{5}{*}{4/2-bit AdamW} & 0.5  & 40.63 & 58.32 & 45.54 & 41.00 & 47.46 \\
                         & 0.6  & 40.66 & 58.20 & 44.97 & 40.67 & 47.20 \\
                         & 0.7  & 40.91 & 60.32 & 46.89 & 42.09 & 47.87 \\
                         & 0.8  & 41.03 & 60.08 & 46.44 & 43.07 & 47.88 \\
                         & 0.9  & 40.92 & 59.61 & 47.12 & 41.53 & 48.53 \\
    \midrule
\multirow{5}{*}{2-bit AdamW}   & 0.5  & 40.36 & 59.85 & 45.88 & 40.47 & 47.21 \\
                         & 0.6  & 40.34 & 58.02 & 45.09 & 40.20 & 46.61 \\
                         & 0.7  & 40.30 & 60.02 & 46.10 & 39.80 & 47.65 \\
                         & 0.8  & 40.21 & 58.32 & 46.10 & 38.20 & 46.81 \\
                         & 0.9  & 39.49 & 57.09 & 45.20 & 39.20 & 46.84 \\
    \bottomrule
\end{tabular}
    \caption{
        The performance of SOLO-quantized low-bit AdamW across various $\beta_1$.
        This validates the necessity of reducing $\beta_1$ to regulate the gradient variance in the ultra-low-precision case.
    }
    \label{table-momentum-adjustment}
\end{table}

\subsection{Visualization of State Changes}

Below, we illustrate the \second state distribution (over 8,192 entries of a linear layer) visualization on other models and tasks.
This empirically corroborates the superiority of SOLO in approximating the dynamics of full-precision optimization.

\begin{figure*}
    \centering
    \includegraphics[width=0.8\linewidth]{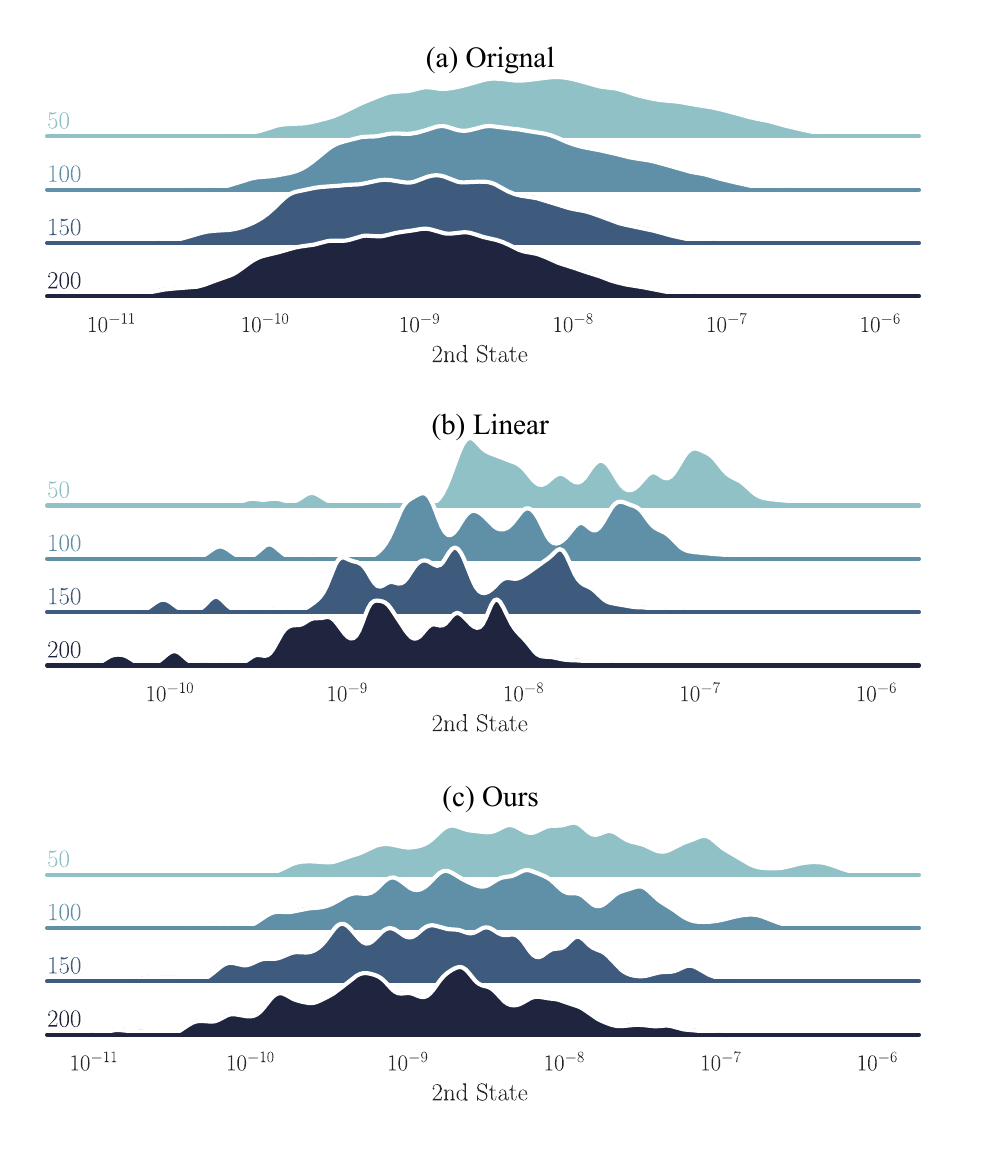}
    \caption{
        How the \second state distribution of Transformer-Base (\includegraphics[width=8pt]{src/flame1.png}) changes at steps 50, 100, 150, and 200.
        DE quantization is omitted, as it leads to an immediate training collapse.
    }
\end{figure*}

\begin{figure*}
    \centering
    \includegraphics[width=0.8\linewidth]{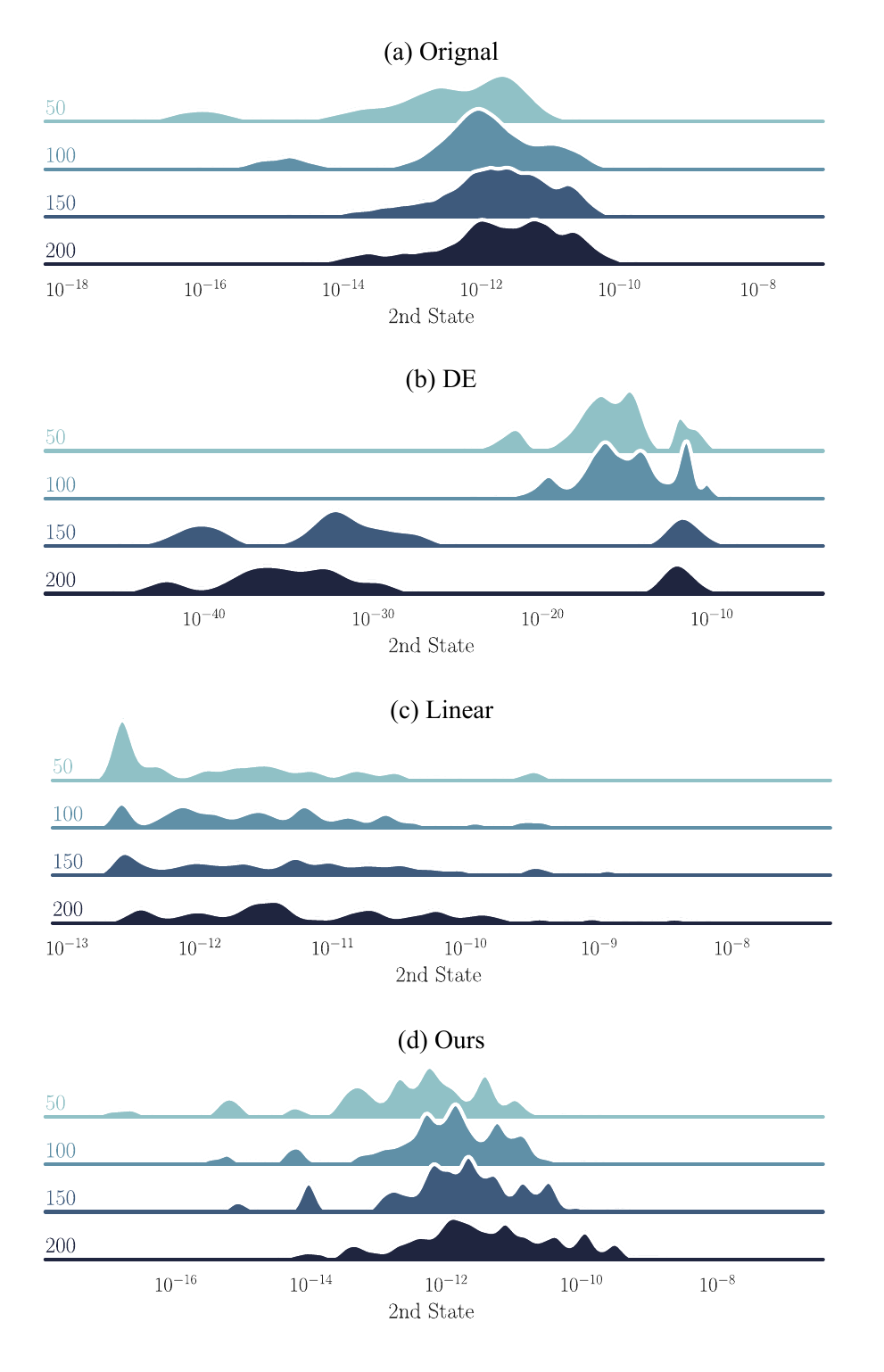}
    \caption{
        How the \second state distribution of RoBERTa-Large (\includegraphics[width=8pt]{src/flame2.png} on COLA) changes at steps 50, 100, 150, and 200.
    }
\end{figure*}

\begin{figure*}
    \centering
    \includegraphics[width=0.8\linewidth]{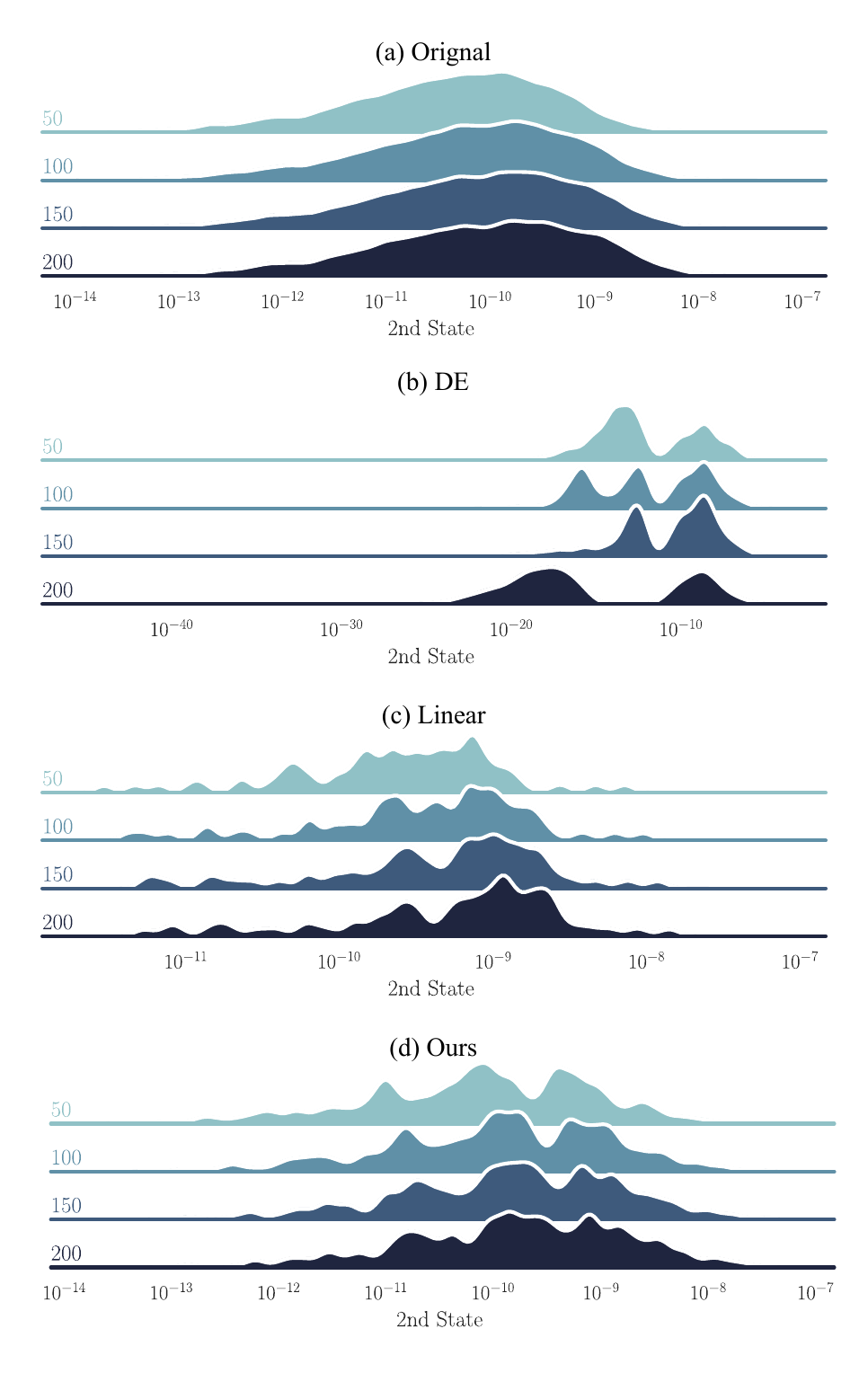}
    \caption{
        How the \second state distribution of DCN (\includegraphics[width=8pt]{src/flame1.png}) changes at steps 50, 100, 150, and 200.
    }
\end{figure*}

\begin{figure*}
    \centering
    \includegraphics[width=0.8\linewidth]{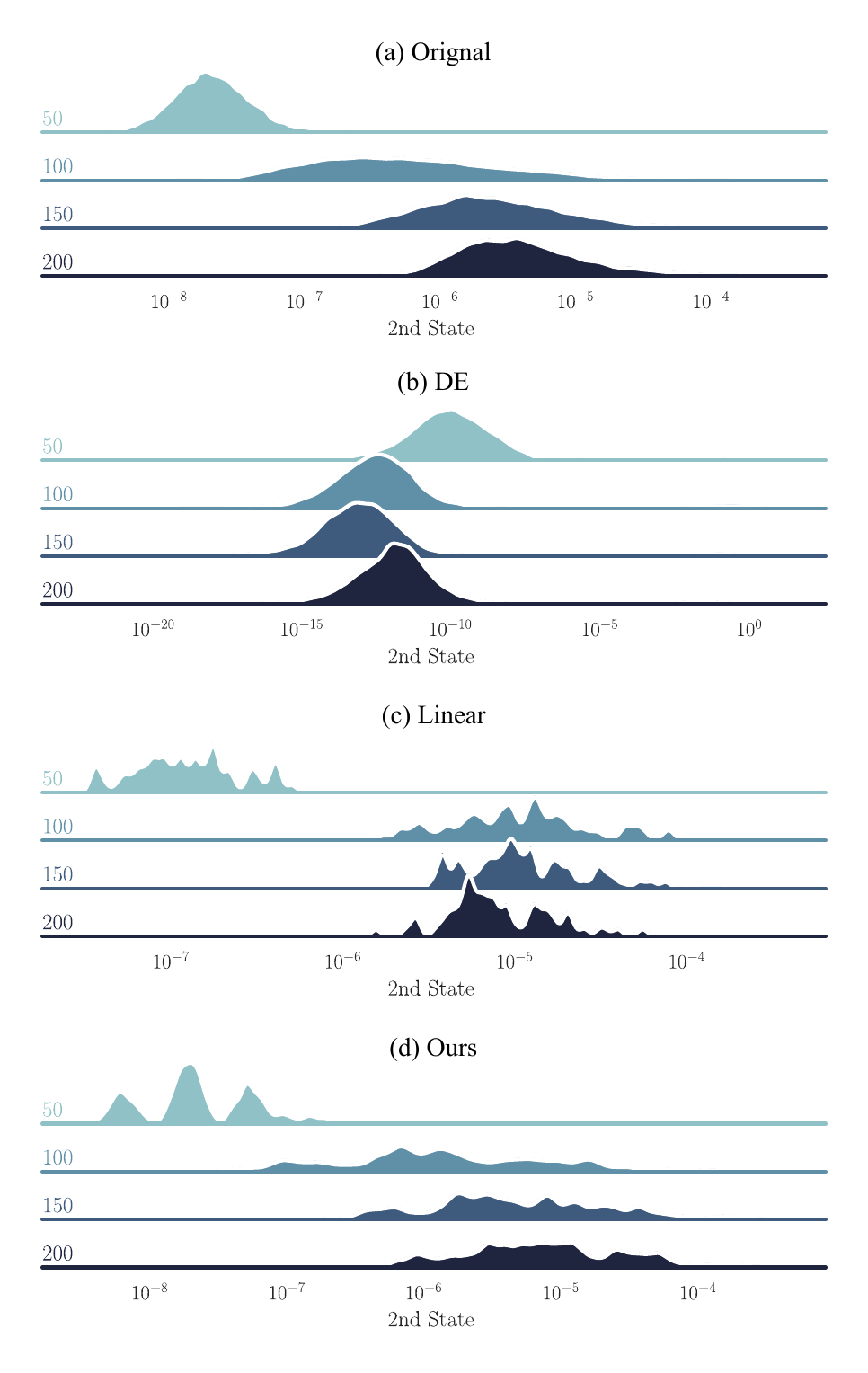}
    \caption{
        How the \second state distribution of HSTU (\includegraphics[width=8pt]{src/flame1.png}) changes at steps 50, 100, 150, and 200.
    }
\end{figure*}

\end{document}